\documentclass{article}

\usepackage[preprint, nonatbib]{nips_2018}

\usepackage[utf8]{inputenc} 
\usepackage[T1]{fontenc}    
\usepackage{hyperref}       
\usepackage{url}            
\usepackage{booktabs}       
\usepackage{amsfonts}       
\usepackage{nicefrac}       
\usepackage{microtype}      
\usepackage{color}
\usepackage{subfig}
\usepackage{bbm}
\usepackage{graphicx}
\usepackage{amsmath}
\usepackage{placeins} 
\usepackage{amsthm}
\usepackage{amssymb}
\usepackage{placeins}
\usepackage{float}
\usepackage{enumitem}

\graphicspath{{figures/}}
\makeatletter 
\newcommand\footnoteref[1]{\protected@xdef\@thefnmark{\ref{#1}}\@footnotemark}
\makeatother

\newcommand{\cX}{\mathcal{X}}
\newcommand{\cY}{\mathcal{Y}}

\newcommand{\xa}{{\mathbf x}^a}
\usepackage{graphicx} 
\usepackage{environ}
\NewEnviron{small_eqn}{%
    \begin{equation}
    \scalebox{0.9}{$\BODY$}
    \end{equation}
    }

\theoremstyle{plain}
\newtheorem{thm}{\protect\theoremname}
\theoremstyle{plain}
\newtheorem{lem}[thm]{\protect\lemmaname}
\theoremstyle{remark}

\makeatother

\providecommand{\lemmaname}{Lemma}
\providecommand{\remarkname}{Remark}
\providecommand{\theoremname}{Theorem}

\title{Variational Learning on Aggregate Outputs\\ with Gaussian Processes}

\author{
Ho Chung Leon Law\thanks{Department of Statistics, Oxford, UK. <ho.law@stats.ox.ac.uk, dino.sejdinovic@stats.ox.ac.uk>} \\
University of Oxford \\
\And
Dino Sejdinovic\footnotemark[1]\\
University of Oxford \\
\And
Ewan Cameron\thanks{Big Data Institute, Oxford, UK. <dr.ewan.cameron@gmail.com, timcdlucas@gmail.com, katherine.battle@bdi.ox.ac.uk>} \\
University of Oxford \\
\And
Tim CD Lucas\footnotemark[2] \\
University of Oxford \\
\And
Seth Flaxman\thanks{Department of Mathematics and Data Science Institute, London, UK. <s.flaxman@imperial.ac.uk>} \\
Imperial College London \\
\And
Katherine Battle\footnotemark[2] \\
University Of Oxford \\
\And
Kenji Fukumizu\thanks{Tachikawa, Japan. <fukumizu@ism.ac.jp>} \\
Institute of Statistical Mathematics \\
}

\begin{document}

\maketitle

\begin{abstract}
While a typical supervised learning framework assumes that the inputs and the outputs are measured at the same levels of granularity, many applications, including global mapping of disease, only have access to outputs at a much coarser level than that of the inputs. Aggregation of outputs makes generalization to new inputs much more difficult. We consider an approach to this problem based on variational learning with a model of output aggregation and Gaussian processes, where aggregation leads to intractability of the standard evidence lower bounds. We propose new bounds and tractable approximations, leading to improved prediction accuracy and scalability to large datasets, while explicitly taking uncertainty into account. We develop a framework which extends to several types of likelihoods, including the Poisson model for aggregated count data. We apply our framework to a challenging and important problem, the fine-scale spatial modelling of malaria incidence, with over $1$ million observations.
\end{abstract}

\section{Introduction}
A typical supervised learning setup assumes existence of a set of input-output examples $\{(x_\ell,y_\ell)\}_\ell$ from which a functional relationship or a conditional probabilistic model of outputs given inputs can be learned. A prototypical use-case is the situation where obtaining outputs $y_\star$ for new, previously unseen, inputs $x_\star$ is costly, i.e., labelling is expensive and requires human intervention, but measurements of inputs are cheap and automated. Similarly, in many applications, due to a much greater cost in acquiring labels, they are only available at a much coarser resolution than the level at which the inputs are available and at which we wish to make predictions. This is the problem of \emph{weakly supervised} learning on aggregate outputs \cite{kueck2005learning, musicant2007supervised}, which has been studied in the literature in a variety of forms, with classification and regression notably being developed separately and without any unified treatment which can allow more flexible observation models. In this contribution, we consider a framework of observation models of aggregated outputs given bagged inputs, which reside in exponential families. While we develop a more general treatment, the main focus in the paper is on the Poisson likelihood for count data, which is motivated by the applications in spatial statistics. In particular, we consider the important problem of fine-scale mapping of diseases. High resolution maps of infectious disease risk can offer a powerful tool for developing National Strategic Plans, allowing accurate stratification of intervention types to areas of greatest impact \cite{gething2016mapping}. In low resource settings these maps must be constructed through probabilistic models linking the limited observational data to a suite of spatial covariates (often from remote-sensing images) describing social, economic, and environmental factors thought to influence exposure to the relevant infectious pathways.  In this paper, we apply our method to 
the incidence of clinical malaria cases. Point incidence data of malaria is typically available at a high temporal frequency (weekly or monthly), but lacks spatial precision, being aggregated by administrative district or by health facility catchment.  The challenge for risk modelling is to produce fine-scale predictions from these coarse incidence data, leveraging the remote-sensing covariates and appropriate regularity conditions to ensure a well-behaved problem. 

Methodologically, the Poisson distribution is a popular choice for modelling count data.  In the mapping setting, the intensity of the Poisson distribution is modelled as a function of spatial and other covariates.  We use Gaussian processes (GPs) as a flexible model for the intensity. GPs are a widely used approach in spatial modelling but also one of the pillars of Bayesian machine learning, enabling predictive models which explicitly quantify their uncertainty.
Recently, we have seen many advances in variational GP posterior approximations, allowing them to couple with more complex observation likelihoods (e.g. binary or Poisson data \cite{Nickisch2008,lloyd2015variational}) as well as a number of effective scalable GP approaches \cite{Quinonero2005,Titsias2009,hensman2013gaussian,hensman2015scalable}, extending the applicability of GPs to dataset sizes previously deemed prohibitive.\\ \\
\textbf{Contribution} Our contributions can be summarised as follows. A general framework is developed for \emph{aggregated observation models} using exponential families and Gaussian processes.  This is novel, as previous work on aggregation or bag models focuses on specific types of output models such as binary classification. Tractable and scalable variational inference methods are proposed for several instances of the aggregated observation models, making use of  novel lower bounds on the model evidence. In experiments, it is demonstrated that the proposed methods can scale to dataset sizes of more than $1$ million observations. We thoroughly investigate an application of the developed methodology to disease mapping from coarse measurements, where the observation model is Poisson, giving encouraging results. Uncertainty quantification, which is explicit in our models, is essential for this application.\\ 
\\
\textbf{Related Work}
The framework of learning from aggregate data was believed to have been first introduced in \cite{musicant2007supervised}, which considers the two regimes of classification and regression. However, while the task of classification of individuals from aggregate data (also known as \emph{learning from label proportions}) has been explored widely in the literature \cite{quadrianto2009estimating, patrini2014almost,kotzias2015group,melnikov2016learning,yu2013propto,yu2014learning, kueck2005learning}, there has been little literature on the analogous regression regime in the machine learning community. Perhaps the closest literature available is \cite{kotzias2015group}, who considers a general framework for learning from aggregate data, but also only considers the classification case for experiments. In this work, we will appropriately adjust the framework in \cite{kotzias2015group} and take this to be our baseline. A related problem arises in the spatial statistics community under the name of `down-scaling', `fine-scale modelling' or `spatial disaggregation' \cite{keil2013downscaling,Howitt2003}, in the analysis of disease mapping, agricultural data, and species distribution modelling, with a variety of proposed methodologies  (cf.~\cite{Xavier2018} and references therein), including kriging \cite{Goovaerts2010}. However, to the best of our knowledge, approaches making use of recent advances in scalable variational inference for GPs are not considered. \\ \\
Another closely related topic is \emph{multiple instance learning} (MIL), concerned with classification with max-aggregation over labels in a bag, i.e. a bag is positively labeled if at least one individual is positive, and it is otherwise negatively labelled. While the task in MIL is typically to predict labels of new unobserved \emph{bags}, \cite{haussmann2017variational} demonstrates that individual labels of a GP classifier can also be inferred in MIL setting with variational inference. Our work parallels that approach, considering bag observation models in exponential families and deriving new approximation bounds for some common generalized linear models. In deriving these bounds, we have taken an approach similar to \cite{lloyd2015variational}, who considers the problem of Gaussian process-modulated Poisson process estimation using variational inference. However, our problem is made more complicated by the aggregation of labels. 
Other related research topics include distribution regression and set regression, as in \cite{szabo2016learning,law2017testing, law2018bayesian} and \cite{zaheer2017deep}. In these regression problems, while the input data for learning is the same as the current setup, the goal is to learn a function at the bag level, rather than the individual level, the application of these methods in our setting, naively treating single individuals as ``distributions'', may lead to suboptimal performance. An overview of some other approaches for classification using bags of instances is given in \cite{Cheplygina2015}.
\vspace{-1em}
\section{Bag observation model: aggregation in mean parameters}
Suppose we have a statistical model $p(y|\eta)$ for output $y\in\cY$, with parameter $\eta$ given by a function of input $x\in\cX$, i.e., $\eta=\eta(x)$. 
Although one can formulate $p(y|\eta)$ in an arbitrary fashion, practitioners often only focus on interpretable simple models, hence we restrict our attention to $p(y|\eta)$ arising from exponential families.  We assume that $\eta$ is the mean parameter of the exponential family.  

Assume that we have a fixed set of points $x_i^a\in\cX$ such that $\xa=\{x^a_1,\ldots,x^a_{N_a}\}$ is a {\it bag} of points with $N_a$ {\it individuals}, and we wish to estimate the regression value $\eta(x_i^a)$ for each individual. However, instead of the typical setup where we have a paired sample $\{(x_\ell,y_\ell)\}_\ell$ of individuals and their outputs to use as a training set, we observe only \emph{aggregate outputs} $y^a$ for each of the bags. Hence, our training data is of the form
\begin{equation}\label{eq:bagged_data}
(\{x^1_i\}_{i=1}^{N_1}, y^1), \dots (\{x^n_i\}_{i=1}^{N_n}, y^n),
\end{equation}
and the goal is to estimate parameters $\eta(x_i^a)$ corresponding to individuals.  To relate the aggregate $y^a$ and the bag $\xa=(x_i^a)_{i=1}^{N_a}$, we use the following {\em bag observation model}: 
\begin{equation}
y^a | \xa \sim p(y|\eta^a), \qquad \eta^a  = \sum_{i=1}^{N_a} w_i^a \eta(x^a_i),
\end{equation}
where $w_i^a$ is an optional fixed non-negative weight used to adjust the scales (see Section \ref{sec:method} for an example). Note that the aggregation in the bag observation model is on the mean parameters for individuals, not necessarily on the individual responses $y_i^a$. This implies that each individual contributes to the mean bag response and that the observation model for bags belongs to the same parametric form as the one for individuals.  For tractable and scalable estimation, we will use variational methods, as the aggregated observation model leads to intractable posteriors.  We consider the Poisson, normal, and exponential distributions, but devote a special focus to the Poisson model in this paper, and refer readers to Appendix \ref{app:general} for other cases and experimental results for the Normal model in Appendix \ref{app:normal_exp}.

It is also worth noting that we place no restrictions on the collection of the individuals, with the bagging process possibly dependent on covariates $x^a_i$ or any unseen factors. The bags can also be of different sizes, with potentially the same individuals appearing in multiple bags. After we obtain our individual model $\eta(x)$, we can use it for prediction of in-bag individuals, as well as out-of-bag individuals. 
\vspace{-1em}
\section{Poisson bag model: Modelling aggregate counts}
\label{sec:method}

The Poisson distribution $p(y|\lambda)=\lambda^y e^{-\lambda}/(y!)$ is considered for count observations, and this paper discusses the Poisson regression with intensity $\lambda(x_i^a)$ multiplied by a `population' $p_i^a$, which is a constant assumed to be known for each individual (or `sub-bag') in the bag. The population for a bag $a$ is given by $p^a=\sum_i p_i^a$.  An observed bag count $y^a$ is assumed to follow 
\[
y^a|\xa \sim {\rm Poisson}(p^a\lambda^a), \quad\lambda^a := \sum_{i=1}^{N_a} \frac{p^a_i}{p^a}  \lambda(x_i^a).
\]
Note that, by introducing unobserved counts $y_i^a\sim {\rm Poisson}(y_i^a|p_i^a\lambda(x_i^a))$, the bag observation $y^a$ has the same distribution as $\sum_{i=1}^{N_a} y_i^a$ since the Poisson distribution is closed under convolutions.  If a bag and its individuals correspond to an area and its partition in geostatistical applications, as in the malaria example in Section \ref{sec:malaria_exp}, the population in the above bag model can be regarded as the population of an area or a sub-area. With this formulation, the goal is to estimate the basic intensity function $\lambda(x)$ from the aggregated observations \eqref{eq:bagged_data}.  Assuming independence given $\{\xa\}_a$, the negative log-likelihood (NLL) $\ell_0$ across bags is 
\begin{equation}
\label{eqn:poiss_nll}
-\log [\Pi_{a=1}^n p(y^a | \mathbf{x}^a)] \overset{c}{=} \sum_{a=1}^n p^a \lambda^a - y^a \log (p^a \lambda^a) \overset{c}{=} \sum_{a=1}^n \left[\sum_{i=1}^{N_a} p^a_i \lambda(x^a_i) - y^a \log \left( \sum_{i=1}^{N_a} p^a_i\lambda(x^a_i) \right)\right],
\end{equation} 
where $\overset{c}{=}$ denotes an equality up to additive constant. 
During training, this term will pass information from the bag level observations $\{y^a\}$ to the individual basic intensity $\lambda(x^a_i)$. It is noted that once we have trained an appropriate model for $\lambda(x^a_i)$, we will be able to make individual level predictions, and also bag level predictions if desired. We will consider baselines with (\ref{eqn:poiss_nll}) using penalized likelihoods inspired by manifold regularization in semi-supervised learning \cite{belkin2006manifold} -- presented in Appendix \ref{sec:baselines}. In the next section, we propose a model for $\lambda$ based on GPs.
\subsection{VBAgg-Poisson: Gaussian processes for aggregate counts}
Suppose now we model $f$ as a Gaussian process (GP), then we have:
\begin{equation}
y^a | \xa \sim {\rm Poisson}\left(\sum_{i=1}^{N_a}p^a_i\lambda^a_i\right), \qquad \lambda^a_i = \Psi(f(x^a_i)), \qquad f \sim GP(\mu, k)
\end{equation}
where $\mu$ and $k$ are some appropriate mean function and covariance kernel $k(x, y)$.
(For implementation, we consider a constant mean function.)
Since the intensity is always non-negative, in all models, we will need to use a transformation $\lambda(x) = \Psi(f(x))$, where $\Psi$ is a non-negative valued function.   
We will consider cases $\Psi(f)=f^2$ and $\Psi(f)=e^f$. A discussion of various choices of this link function in the context of Poisson intensities modulated by GPs is given in \cite{lloyd2015variational}. Modelling $f$ with a GP allows us to propagate uncertainty on the predictions to $\lambda^a_i$, which is especially important in this weakly supervised problem setting, where we do not directly observe any individual output  $y^a_i$. Since the total number of individuals in our target application of disease mapping is typically in the millions (see Section \ref{sec:malaria_exp}), we will approximate the posterior over $\lambda^a_i:=\lambda(x^a_i)$ using variational inference, with details found in Appendix \ref{app:poisson_details}.

For scalability of the GP method, as in previous literature \cite{haussmann2017variational,lloyd2015variational}, we use a set of inducing points $\{u_\ell\}_{\ell=1}^m$, which are given by the function evaluations of the Gaussian process $f$ at landmark points $W=\{w_1,\ldots,w_m\}$; i.e., $u_\ell = f(w_\ell)$. The distribution $p(u|W)$ is thus given by 
\begin{equation}
u\sim N(\mu_W,K_{WW}),  \qquad \mu_W=(\mu(w_\ell))_\ell, \quad K_{WW} = (k(w_s,w_t))_{s,t}. 
\end{equation}
The joint likelihood is given by:
\begin{equation}\label{eq:likelihood}
p(y,f,u|X,W,\Theta) = \prod_{a=1}^n \prod_{i=1}^{N_a} {\rm Poisson}(y^a| p^a\lambda^a) p(f| u)p(u|W), \text{ with } f|u\sim  GP(\tilde{\mu}_u,\tilde{K}),
\end{equation}

\begin{equation}
\label{eqn:f|u}
	\tilde{\mu}(z)= \mu_z + {\bf k}_{zW}K_{WW}^{-1} (u - \mu_W), \quad \tilde{K}(z,z')=k(z,z')-{\bf k}_{zW}K_{WW}^{-1} {\bf k}_{Wz'}
\end{equation}
where ${\bf k}_{zW}=(k(z,w_1),\ldots,k(z,w_\ell))^T$, with $\mu_W$, $\mu_z$ denoting their respective evaluations of the mean function $\mu$ and $\Theta$ being parameters of the mean and kernel functions of the GP. Proceeding similarly to \cite{lloyd2015variational}, which discusses (non-bag) Poisson regression with GP, we obtain a lower bound of the marginal log-likelihood $\log p(y|\Theta)$:
\begin{align}
\log p(y|\Theta) & = \log \int \int p(y,f,u|X,W,\Theta) dfdu \nonumber \\
& \geq  \int \int \log\Bigl\{p(y|f,\Theta)\frac{p(u|W)}{q(u)}\Bigr\}p(f|u,\Theta) q(u)dfdu \quad \text{(Jensen's inequality)}\nonumber \\
& = \sum_a \int \int \Bigl\{ y^a \log\Bigl(\sum_{i=1}^{N_a} p^a_i \Psi(f(x^a_i)\Bigr) - \Bigl(\sum_{i=1}^{N_a} p^a_i\Psi(f(x^a_i))\Bigr) \Bigr\}  
p(f| u) q(u)dfdu \nonumber \\
& \qquad - \sum_a \log(y^a!) - KL(q(u)||p(u|W))=:\mathcal L(q,\Theta),
\label{eq:VB_obj}
\end{align}
where $q(u)$ is a variational distribution to be optimized. 
The general solution to the maximization over $q$ of the evidence lower bound $\mathcal L(q,\Theta)$ above is given by the posterior of the inducing points $p(u|y)$, which is intractable.  We introduce a restriction to the class of $q(u)$ to approximate the posterior $p(u|y)$.  Suppose that the variational distribution $q(u)$ is Gaussian, $q(u) = N(\eta_u, \Sigma_u)$.  We then need to maximize the lower bound $\mathcal L(q,\Theta)$ over the variational parameters $\eta_u$ and $\Sigma_u$. 

The resulting $q(u)$ gives an approximation to the posterior $p(u|y)$ which also leads to a Gaussian approximation $q(f) = \int p(f|u)q(u)du$ to the posterior $p(f|y)$, which we finally then transform through $\Psi$ to obtain the desired approximate posterior on each $\lambda(x_a^i)$ (which is either log-normal or non-central $\chi^2$ depending on the form of $\Psi$). The approximate posterior on $\lambda$ will then allow us to make predictions for individuals while, crucially, taking into account the uncertainties in $f$ (note that even the posterior predictive mean of $\lambda$ will depend on the predictive variance in $f$ due to the nonlinearity $\Psi$). We also want to emphasis the use of inducing variables is essential for scalability in our model: we cannot directly obtain approximations to the posterior of $\lambda(x^a_i)$ for all individuals, since this is often large in our problem setting (Section \ref{sec:malaria_exp}).

As the $p(u|W)$ and $q(u)$ are both Gaussian, the last term (KL-divergence) of (\ref{eq:VB_obj}) can be computed explicitly with exact form found in Appendix \ref{app:kl_term}.
To consider the first two terms, let  $q^a(v^a)$ be the marginal normal distribution of $v^a=(f(x^a_1),\ldots,f(x^a_{N_a}))$, where $f$ follows the variational posterior  $q(f)$.
The distribution of $v^a$ is then $N(m^a,S^a)$, using \eqref{eqn:f|u} :  
\begin{equation}
{\small
\label{eqn:meanCov}
m^a=\mu_{\xa} + K_{\xa W}K_{WW}^{-1} (\eta_u-\mu_W), \;S^a =K_{\xa,\xa}-K_{\xa W}\left(K_{WW}^{-1}-K_{WW}^{-1} \Sigma_u K_{WW}^{-1}\right) K_{W\xa}
}
\end{equation}

In the first term of (\ref{eq:VB_obj}), each summand is
of the form 
\begin{equation}\label{eq:lse}
y^a  \int \log\Bigl(\sum_{i=1}^{N_a} p^a_i \Psi\left({v_i^a}\right)\Bigr) q^a(v^a)dv^a -  \sum_{i=1}^{N_a}p^a_i\int  \Psi\left({v_i^a}\right) q^a(v^a)dv^a,
\end{equation}

in which the second term is tractable for both of $\Psi(f)=f^2$ and $\Psi(f)=e^f$.  The integral of the first term, however with $q^a$ Gaussian is not tractable.
To solve this, we take different approaches for $\Psi(f)=f^2$ and $\Psi(f)=e^f$; for the former, approximation by Taylor expansion is applied, while for the latter, further lower bound is taken. 

First consider the case $\Psi(f)=f^2$, and rewrite the first term of (\ref{eq:VB_obj}) as:
\begin{equation*}
y^a  \mathbb{E}\log\left\Vert V^a\right\Vert ^{2} \quad, \text{where}\; V^a \sim N( \tilde{m}^a, \tilde{S}^a),
\end{equation*}
with $P^a = diag\left(p^a_1,\dots, p^a_{N_a}\right), \tilde{m}^a = {P^a}^{1/2} m^a$ and $\tilde{S}^a = {P^a}^{1/2} S^a{P^a}^{1/2}$. By a Taylor series approximation for $\mathbb{E}\log\left\Vert V^a\right\Vert ^{2}$ (similar to \cite{teh2007collapsed}) around $\mathbb{E}\left\Vert V^a\right\Vert ^{2}=\left\Vert \tilde{m}^a \right\Vert ^{2}+tr\tilde{S}^a$, we obtain 
\begin{multline}
\int \log\Bigl(\sum_{i=1}^{N_a} p^a_i (v_i^a)^2 \Bigr) q^a(v^a)dv^a  \\
\approx  \log\left(m^{a\top}P^a m^a +tr(S^a P^a)\right)-\frac{2m^{a\top}P^aS^aP^am^a+tr\Bigl((S^a P^a)^2\Bigr)}{\left(m^{a\top}P^a m^a +tr(S^a P^a)\right)^{2}}=:\zeta^a.
\end{multline}
with details are in Appendix \ref{sec:taylor}. An alternative approach which we use for the case $\Psi(f)=e^f$ is to take a further lower bound, which is applicable to a general class of $\Psi$ (we provide further details for the analogous approach for $\Psi(v)=v^2$ in Appendix \ref{sec:lb_square}). We use the following Lemma (proof found in Appendix \ref{app:lemma}):
\begin{lem}
\label{lem:logsum}
Let $v=[v_{1},\ldots,v_{N}]^{\top}$ be a random vector with probability density $q(v)$ with marginal densities $q_i(v)$, and let $w_{i}\geq0$, $i=1,\ldots,N$. Then, for any non-negative valued function $\Psi(v)$,
\[
\int\log\bigl(\sum_{i=1}^{N}w_{i}\Psi(v_{i})\bigr)q(v)dv\geq\log\Bigl(\sum_{i=1}^{N}w_{i}e^{\xi_{i}}\Bigr), \quad{where}\quad
\xi_{i}:=\int\log\Psi(v_{i})q_{i}(v_{i})dv_{i}.
\]
\end{lem}
Hence we obtain that 
\begin{equation}
\label{eq:lb_lse} 
\int \log\bigl(\sum_{i=1}^{N_a} p^a_i e^{v^a_i} \bigr) q^a(v^a)dv^a
\geq  \log\Bigl(\sum_{i=1}^{N_a} p^a_i e^{m^a_i}\Bigr),
\end{equation}
 
Using the above two approximation schemes, our objective (up to constant terms) can be formulated as: 
1) $\Psi(v)=v^2$
\begin{equation}
\mathcal{L}_1^s(\Theta,\eta_u,\Sigma_u,W) := \sum_{a=1}^n y^a \zeta^a
- \sum_{a=1}^n \sum_{i=1}^{N_a} \bigl\{(m^a_i)^2 + S^a_{ii}/2\bigr\} - KL(q(u)||p(u|W)),
\end{equation}
2) $\Psi(v)=e^v$
\begin{equation}
\mathcal{L}_1^e(\Theta,\eta_u,\Sigma_u,W) := \sum_{a=1}^n y^a \log\bigl(\sum_{i=1}^{N_a} e^{m^a_i}\bigr) - \sum_{j=1}^n \sum_{i=1}^{N_a} e^{m^a_i + S^a_{ii}/2} 
- KL(q(u)||p(u|W)).
\end{equation}
Given these objectives, we can now optimise these lower bounds with respect to variational parameters $\{\eta_u, \Sigma_u\}$, parameters $\Theta$ of the mean and kernel functions, using stochastic gradient descent (SGD) on bags. Additionally, we might also learn $W$, locations for the landmark points. In this form, we can also see that the bound for $\Psi(v)=e^v$ has the added computational advantage of not requiring the full computation of the matrix $S^a$, but only its diagonals, while for $\Psi(v)=v^2$ computation of $\zeta^a$ involves full $S^a$, which may be problematic for extremely large bag sizes.
\section{Experiments}
\label{sec:exp}
We will now demonstrate various approaches: Variational Bayes with Gaussian Process (VBAgg), a MAP estimator of Bayesian Poisson regression with explicit feature maps (Nystr\"om) and a neural network (NN) -- the latter two employing manifold regularisation with RBF kernel (unless stated otherwise). For additional baselines, we consider a constant within bag model (constant), i.e. $\hat{\lambda^a_i} = \frac{y^a}{p^a}$ and also consider creating `individual' covariates by aggregation of the covariates within a bag (bag-pixel). For details of all these approaches, see Appendix \ref{sec:baselines}. We also denote $\Psi(v) = e^v$ and $v^2$ as Exp and Sq respectively.

We implement our models in \textit{TensorFlow}\footnote{Code will be available for use.} and use SGD with Adam \cite{kingma2014adam} to optimise their respective objectives, and we split the dataset into $4$ parts, namely train, early-stop, validation and test set. Here the early-stop set is used for early stopping for the Nystr\"om, NN and bag-pixel models, while the VBAgg approach ignores this partition as it optimises the lower bound to the marginal likelihood. The validation set is used for parameter tuning of any regularisation scaling, as well as learning rate, layer size and multiple initialisations. Throughout, VBAgg and Nystr\"om have access to the same set of landmarks for fair comparison. It is also important to highlight that we perform early stopping and tuning based on \textit{bag} level performance on NLL only, as this is the only information available to us. 

For the VBAgg model, there are two approaches to tuning, one approach is to choose parameters based on NLL on the validation bag sets, another approach is to select all parameters based on the training objective $\mathcal{L}_1$, the lower bound to the marginal likelihood. We denote the latter approach VBAgg-Obj and report its toy experimental results in Appendix \ref{app:poisson_toy} for presentation purposes. In general, the results are relatively \textit{insensitive} to this choice, especially when $\Psi(v)=v^2$. To make predictions, we use the mean of our approximated posterior (provided by a log-normal and non-central $\chi^2$ distribution for Exp and Sq). As an additional evaluation, we report mean square error (MSE) and bag performance results in Appendix \ref{app:experiments}.
\begin{figure}
\centering
\includegraphics[width=\textwidth]{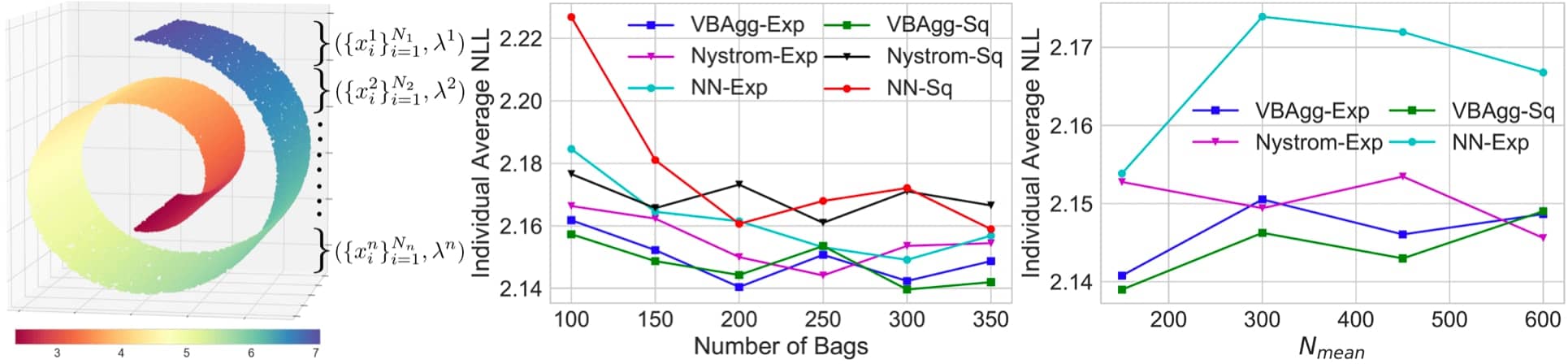} 
  \caption{\textbf{Left}: Random samples on the Swiss roll manifold. \textbf{Middle, Right}: Individual Average NLL on train set for varying number of training bags $n$ and increasing $N_{mean}$, over $5$ repetitions. Constant prediction within bag gives a NLL of $2.22$. bag-pixel model gives NLL above $2.4$ for the varying number of bags experiment.}%
    \label{fig:swiss}%
\end{figure}

\subsection{Poisson Model: Swiss Roll}
We first demonstrate our method on the swiss roll dataset\footnote{\label{fn:data}The swiss roll manifold function (for sampling) can be found on the Python \textit{scikit-learn} package.}, illustrated in Figure \ref{fig:swiss} (left). To make this an aggregate learning problem, we first construct $n$ bags with sizes drawn from a negative binomial distribution $N_a \sim NB(N_{mean}, N_{std})$, where $N_{mean}$ and $N_{std}$ represents the respective mean and standard deviation of $N_a$.
We then randomly select $\sum_{a=1}^n N_a$ points from the swiss roll manifold to be the locations, giving us a set of colored locations in $\mathbb{R}^3$. Ordering these random locations by their $z$-axis coordinate, we group them, filling up each bag in turn as we move along the $z$-axis. 
The aim of this is to simulate that in real life the partitioning of locations into bags is often not independent of covariates. Taking the colour of each location as the underlying rate $\lambda^a_i$ at that location, we simulate $y^a_i \sim Poisson(\lambda^a_i)$, and take our observed outputs to be $y^a = \sum_{i=1}^{N_a} y^a_i \sim Poisson(\lambda^a)$, where $\lambda^a = \sum_{i=1}^{N_a} \lambda^a_i$. Our goal is then to predict the underlying individual rate parameter $\lambda^a_i$, given only bag-level observations $y^a$. To make this problem even more challenging, we embed the data manifold into $\mathbb{R}^{18}$ by rotating it with a random orthogonal matrix. For the choice of $k$ for VBAgg and Nystr\"om, we use the RBF kernel, with the bandwidth parameter learnt. For landmark locations, we use the K-means++ algorithm, so that landmark points lie evenly across the data manifold.
\vspace{-1em} \\
\paragraph{Varying number of Bags: $n$}
To see the effect of increasing number of bags available for training, we fix $N_{mean}=150$ and $N_{std}= 50$, and vary the number of bags $n$ for the training set from $100$ to $350$ with the same number of bags for early stopping and validation. Each experiment is repeated for $5$ runs, and results are shown in Figure \ref{fig:swiss} for individual NLL on the train set. Again we emphasise that the individual labels are not used in training. We see that all versions of VBAgg outperform all other models, in terms of MSE and NLL, with statistical significance confirmed by a signed rank permutation test (see Appendix \ref{app:poisson_toy}). We also observe that the bag-pixel model has poor performance, as a result of losing individual level covariate information in training by simply aggregating them.\\ \\\textbf{Varying number of individuals per bag: $N_{mean}$} \quad To study the effect of increasing bag sizes (with larger bag sizes, we expect "disaggregation" to be more difficult), we fix the number of training bags to be $600$ with early stopping and validation set to be $150$ bags, while varying the number of individuals per bag through $N_{mean}$ and $N_{std}$ in the negative binomial distribution. To keep the relative scales between $N_{mean}$ and $N_{std}$ the same, we take $N_{std} = N_{mean}/2$. The results are shown in Figure \ref{fig:swiss}, focusing on the best performing methods in the previous experiment. Here, we observe that VBAgg models again perform better than the Nystr\"om and NN models with statistical significance as reported in Appendix \ref{app:poisson_toy}, with performance stable as $N_{mean}$ increases. \\ \\
\textbf{Discussion} \quad To gain more insight into the VBAgg model, we look at the calibration of our two different Bayesian models: VBAgg-Exp and VBAgg-Square. We compute their respective posterior quantiles and observe the ratio of times the true $\lambda^a_i$ lie in these quantiles. We present these in Appendix \ref{app:poisson_toy}. The calibration plots reveal an interesting nature about using the two different approximations for using $e^v$ versus $v^2$ for $\Psi(v)$. While experiments showed that the two model perform similarly in terms of NLL, the calibration of the models is very different. While the VBAgg-Square is well calibrated in general, the VBAgg-Exp suffers from poor calibration. This is not surprising, as VBAgg-Exp uses an additional lower bound on model evidence. Thus, uncertainty estimates given by VBAgg-Exp should be treated with care.

\subsection{Malaria Incidence Prediction}
\label{sec:malaria_exp}
We now demonstrate the proposed methodology on an important real life malaria prediction problem for an endemic country from the Malaria Atlas Project database\footnote{Due to confidentiality reasons, we do not report country or plot the full map of our results.}. In this problem, we would like to predict the underlying malaria incidence rate in each $1$km by $1$km region (referred to as a pixel), while having only observed aggregated incidences of malaria $y^a$ at much larger regional levels, which are treated as bags of pixels. These bags are non-overlapping administrative units, with $N_a$ pixels per bag ranging from 13 to 6,667, with a total of 1,044,683 pixels. In total, data is available for $957$ bags\footnote{We consider $576$ bags for train, $95$ bags each for validation and early-stop, with $191$ bags for testing, with different splits across different trials, selecting them to ensure distributions of labels are similar across sets.}. Along with these pixels, we also have population estimates $p^a_i$ (per $1000$ people) for pixel $i$ in bag $a$, spatial coordinates given by $s^a_i$, as well as covariates $x^a_i \in \mathbb{R}^{18}$, collected by remote sensing. Some examples of covariates includes accessibility, distance to water, mean of land surface temperature and stable night lights. It is clear that rather than expecting malaria incidence rate to be constant throughout the entire bag (as in Figure \ref{fig:malaria}), we expect pixel incidence rate to vary, depending on social, economic and environmental factors \cite{weiss2015re}. Our goal is therefore to build models that can predict malaria incidence rates at a \textit{pixel} level.

We assume a Poisson model on each individual pixel, i.e.  $y^a \sim Poisson(\sum_i p_i^a \lambda_i^a)$, where $\lambda_i^a$ is the underlying pixel incidence rate of malaria per $1000$ people that we are interested in predicting. We consider the VBAgg, Nystr\"om and NN as prediction models and use a kernel given as a sum of an ARD (automatic relevance determination) kernel on covariates and a Mat\'ern kernel on spatial locations for the VBAgg and Nystr\"om methods, learning all kernel parameters (the kernel expression is provided in Appendix \ref{app:malaria}). We use the same kernel for manifold regularisation in the NN model. This kernel choice incorporates spatial information, while allowing feature selection amongst other covariates. For choice of landmarks, we ensure landmarks are placed evenly throughout space by using one landmark point per training bag (selected by k-means++). This is so that the uncertainty estimates we obtain are not too sensitive to the choice of landmarks.
\begin{figure}
\centering
\includegraphics[width=\linewidth]{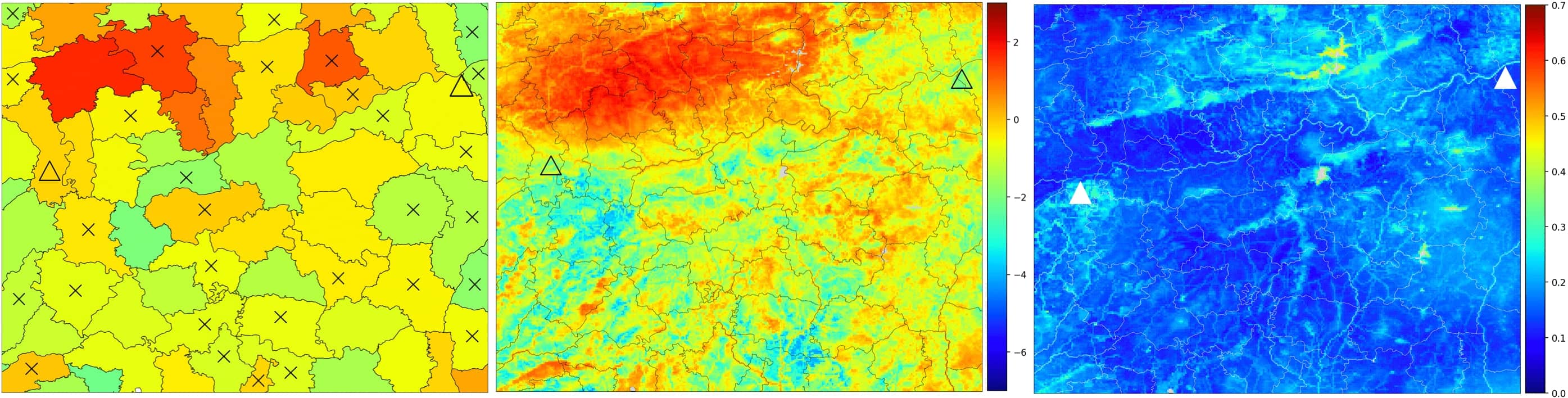}
\captionof{figure}{Triangle denotes approximate start and end of river location, crosses denotes non-train set bags. Malaria incidence rate $\lambda_i^a$ is per $1000$ people. \textbf{Left, Middle}:  $
\log(\hat{\lambda}^a_i)$, with constant model (Left), and VBAgg-Obj-Sq (tuned on $\mathcal{L}_1^s$) (Middle). \textbf{Right}: Standard deviation of the posterior $v$ in (\ref{eqn:meanCov}) with VBAgg-Obj-Sq.}
\label{fig:malaria}
\end{figure}
In this problem, no individual-level labels are available, so we report Bag NLL and MSE (on observed incidences) on the test bags in Appendix \ref{app:malaria} over $10$ different re-splits of the data. Although we can see that Nystr\"om is the best performing method, the improvement over VBAgg models is not statistically significant. On the other hand, both VBAgg and Nystr\"om models statistically significantly outperform NN, which also has some instability in its predictions, as discussed in Appendix \ref{app:malaria_preds}. However, a caution should be exercised when using the measure of performance at the bag level as a surrogate for the measure of performance at the individual level: in order to perform well at the bag level, one can simply utilise spatial coordinates and ignore other covariates, as malaria intensity appears to smoothly vary between the bags (Left of Figure \ref{fig:malaria}). However, we do not expect this to be true at the individual level. 

To further investigate this, we consider a particular region, and look at the predicted individual malaria incidence rate, with results found in Figure \ref{fig:malaria} and in Appendix \ref{app:malaria_preds} across $3$ different data splits, where the behaviours of each of these models can be observed. While Nystr\"om and VBAgg methods both provide good bag-level performance, Nystr\"om and VBAgg-Exp can sometimes provide overly-smooth spatial patterns, which does not seem to be the case for the VBAgg-Sq method (recall that VBAgg-Sq performed best in both prediction and calibration for the toy experiments). In particular, VBAgg-Sq consistently predicts higher intensity along rivers (a known factor \cite{warrel2017oxford}; indicated by triangles in Figure \ref{fig:malaria}) using only coarse aggregated intensities, demonstrating that prediction of (unobserved) pixel-level intensities is possible using fine-scale environmental covariates, especially ones known to be relevant such as covariates indicated by the Topographic Wetness Index, a measure of wetness, see Appendix \ref{app:river} for more details.

In summary, by optimising the lower bound to the marginal likelihood, the proposed variational methods are able to learn useful relations between the covariates and pixel level intensities, while avoiding the issue of overfitting to spatial coordinates. Furthermore, they also give uncertainty estimates (Figure \ref{fig:malaria}, right), which are essential for problems like these, where validation of predictions is difficult, but they may guide policy and planning.

\section{Conclusion}
Motivated by the vitally important problem of malaria, which is
the direct cause of around 187 million clinical cases \cite{bhatt2015effect} and 631,000 deaths \cite{gething2016mapping} each year in sub-Saharan Africa, we have proposed a general framework of \emph{aggregated observation models} using Gaussian processes, along with scalable variational methods for inference in those models, making them applicable to large datasets. The proposed method allows learning in situations where outputs of interest are available at a much coarser level than that of the inputs, while explicitly quantifying uncertainty of predictions. 
The recent uptake of digital health information systems offers a wealth of new data which is abstracted to the aggregate or regional levels to preserve patient anonymity. The volume of this data, as well as the availability of much more granular covariates provided by remote sensing and other geospatially tagged data sources, allows to probabilistically disaggregate outputs of interest for finer risk stratification, e.g.~assisting public health agencies to plan the delivery of disease interventions. This task demands new high-performance machine learning methods and we see those that we have developed here as an important step in this direction.

\section*{Acknowledgement}
We thank Kaspar Martens for useful discussions, and Dougal Sutherland for providing the code base in which this work was based on. HCLL is supported by the EPSRC and MRC through the OxWaSP CDT programme (EP/L016710/1). HCLL and KF are supported by JSPS KAKENHI 26280009. EC and KB are supported by OPP1152978, TL by OPP1132415 and the MAP database by OPP1106023. DS is supported in part by the ERC (FP7/617071) and by The Alan Turing Institute (EP/N510129/1). The data were provided by the Malaria Atlas Project supported by the Bill and Melinda Gates Foundation.  

\clearpage
\bibliography{references}

\FloatBarrier
\appendix

\newpage
\section{Aggregated Exponential Family Models}
\label{app:general}
Consider an observation model of the form
\begin{equation}
p(y|\theta) = \exp\left (\frac{y\theta - c(\theta)}{\tau} \right)h(y,\tau),
\end{equation}
where response $y$ is one-dimensional, $\theta$ is a natural parameter corresponding to the statistic $y$, $\tau$ is a dispersion parameter, and $h(y,\tau)$ is base measure.  For simplicity, we will assume that natural parameters corresponding to the other parts of the sufficient statistic are fixed and folded into the base measure. Let $\eta$ be the corresponding mean parameter, i.e. 
\[
	\eta = \mathbb E_{\theta} y = \int y p(y|\theta) dy
\]
and $\theta = F(\eta)$ be the link function mapping from mean to the natural parameters and $G(\theta)$ its inverse. We wish to model the mean parameter $\eta=\eta(x)$ using a Gaussian process on a domain $\cX$ together with a function $\Psi$ which transforms the GP value to the natural parameter space, i.e.
\begin{equation}
\eta(x) = \Psi(f(x)), \qquad f\sim \mathcal{GP}(\mu,k).
\end{equation}
For example, the mean parameter for some models is restricted to the positive part of the real line, while the GP values cover the whole real line.  
We will consider the following examples:
\begin{itemize}
\item {\bf Normal} (with fixed variance).  $F=G=idenity$ and $\Psi$ can be identity, too, as there are no restrictions on the mean parameter space. 
\item {\bf Poisson}.  $F(\eta)=\log \eta$, $G(\theta)=e^\theta$.  $\Psi$ should take a positive value, so we consider $\Psi(v)=e^v$ or $\Psi(v)=v^2$.
\item {\bf Exponential}.  $p(y|\eta)=\exp(- y/\eta)/\eta$ and $\theta=-\eta$, $F(\eta)=-1/\eta$, $G(\theta)=-1/\theta$. $\Psi$ should take a positive value, so we consider $\Psi(v)=e^v$ or $\Psi(v)=v^2$
\end{itemize}
Note that the link function $F$ is concave for all the examples above.

\subsection{Bag model}

We will consider the aggregation in the mean parameter space. Namely, let $y^1,\ldots,y^n$ be $n$ independent aggregate responses for each of the $n$ bags of covariates $\xa=\{x^a_1,\ldots,x^a_{N_a}\}$, $a=1,\ldots,n$. We assume the following aggregation model:
\begin{equation}
	y^a\sim p(y|\eta_a), \quad \eta^a = \sum_{i=1}^{N_a} w^a_i \eta^a_i = \sum_{i=1}^{N_a} w^a_i \Psi(f(x^a_i)),\quad a=1,\ldots,n.
\end{equation}
where $w^a_i$ are fixed weights to adjust the scales among the individuals and the bag (e.g., adjusting for population size).  

We also can model individual (unobserved) variables $y^a_i$ ($i=1,\ldots,N_a)$, which follow:
\begin{equation}
y^a_i \sim p(y|\eta^a_i), \qquad \eta^a_i = \Psi(f(x^a_i)), \quad i=1,\ldots,N_a,\;a=1,\ldots,n.
\end{equation}

Note that we consider aggregation in mean parameters of responses, not in the responses themselves. If we consider a case where underlying individual responses $y^a_i$ aggregate to $y^a$ as a weighted sum, the form of the bag likelihood and individual likelihood would be different unless we restrict attention to distribution families which are closed under both scaling and convolution. However, when aggregation occurs in the mean parameter space, the form of the bag likelihood and individual likelihood is always the same. This corresponds to the following measurement process:
\begin{itemize}
\item Each individual has a mean parameter $\eta^{a}_i$ - if it were possible to sample a response for that particular individual, we would obtain a sample $y^a_i \sim p(\cdot | \eta^a_i)$
\item However, we cannot sample the individual and we can only observe a bag response. But in that case, only a single bag response is taken and depends on all individuals simultaneously. Each individual contributes in terms of an increase in a mean bag response, but this measurement process is different from the two-stage procedure by which we aggregate individual responses. 
\end{itemize}

\subsection{Marginal likelihood and ELBO}

Let $Y=(y^1,\ldots,y^n)$ (bag observations).  With the inducing points $u=f(W)$, the marginal likelihood is 
\begin{equation}
p(Y)= \int\int \prod_{a=1}^n p(y^a|\eta^a)p(f|u)p(u)dudf.
\end{equation}
The evidence lower bound can be derived as
\begin{align}
\log p(Y) & = \log \int\int \Bigl\{ \prod_{a=1}^n p(y^a|\eta^a)\frac{p(u)}{q(u)}\Bigr\} p(f|u)q(u)dudf \nonumber \\
& \geq \int\int \log \Bigl\{ \prod_{a=1}^n p(y^a|\eta^a)\frac{p(u)}{q(u)}\Bigr\} p(f|u)q(u)dudf \nonumber \\
& = \sum_{a=1}^n \frac{y^a}{\tau}\int F\Bigl(\sum_i w^a_i \Psi(f(x^a_i)) \Bigr)
q(f) df - \int c\Bigl(F\Bigl(\sum_i w^a_i \Psi(f(x^a_i)) \Bigr)\Bigr) q(f)df \nonumber \\
& \qquad - \int q(u)\log \frac{q(u)}{p(u)}du,
\end{align}
where $q(f) = \int p(f|u)q(u)du$.

By setting the variational distribution $q(u)$ as Gaussian, the third term is tractable.  The first and second terms are however tractable only in limited cases.  
The cases we develop are the Poisson bag model, described in the main text, as well as the normal bag model and the exponential bag model, described below.

\subsection{Normal bag model}
\label{app:normal}
$F$ is identity and $c(\theta)=\theta^2/2$, which makes both the first and the second terms tractable with the choice of $\Psi(v)=v$. Moreover, the viewpoints of aggregating in the mean parameters and in the individual responses are equivalent for this model and we can also allow different variance parameters for different bags (and individuals).

Consider a bag $a$ of items $\{x_i^a\}_{i=1}^{N_a}$. Each item $x_i^a$ is assumed to have a weight $w_i^a$. At the individual level, we model the (unobserved) responses $y_i^a$ as
\begin{equation}
 y^a_i | x^a_i \sim \mathcal{N}\left(w^a_i\mu^a_i, \left(w^a_i\right)^2 \tau^a_i\right)
\end{equation}
where $\mu^a_i = \mu(x^a_i)$, thus $\mu^a_i$ is a \emph{mean parameter per unit weight} corresponding to the item $x^a_i$ and it is assumed to be a function of both $x^a_i$. Similarly, $\tau^a_i$ is a variance parameter per unit weight. 
At the bag level, we consider the following model for the observed aggregate response $y^a$, assuming conditional independence of individual responses given covariates $\xa=\{x_1^a,\ldots,x_{N_a}^a\}$: 
\begin{equation}
y^a  = \sum_{i=1}^{N_a} y^a_i, \;\text{i.e.}\; y^a | \mathbf{x}^a \sim \mathcal{N}(w^a \mu^a, (w^a)^2 \tau^a), \qquad \mu^a = \sum_{i=1}^{N_a} \frac{w^a_i}{w^a} \mu^a_i, \tau^a =  \frac{\sum_{i=1}^{N_a}(w^a_i)^2{\tau}^a_i}{(w^{a})^2} 
\end{equation}
where $\mu^a$ and $\tau^a$ are the mean and variance parameters per unit weight of the whole bag $a$ and $w^a = \sum_{i=1}^{N_a} w^a_i$ is the \emph{total weight} of bag $a$. Although we can take $\tau^a_i$ to also be a function of the covariates, here for simplicity, we take $\tau^a_i = \tau^a$ to be constant per bag (note the abuse of notation). We can now compute the negative log-likelihood (NLL) across bags (assuming conditional independence given the $\mathbf{x}^a$):
\begin{equation}
\ell_0 =  -\log \left[\Pi_{a=1}^n p(y^a | \mathbf{x}^a)\right] {=} \frac{1}{2} \sum_{a=1}^n \left\{\log \left ( 2\pi \tau^a \sum_{i=1}^{N_a}(w^a_i)^2\right) +  \frac{\left(y^a - \sum_{i=1}^{N_a} w^a_i \mu^a_i\right)^2}{ \sum_{i=1}^{N_a}(w^a_i)^2 \tau^a}\right\}
 \end{equation}
where $\mu^a_i = f(x^a_i)$ is the function we are interested in, and $\tau^a$ are the variance parameters to be learnt. 

We can now consider the lower bound to the marginal likelihood as below (assuming $w_i^a=1$ here to simplify notation, while the analogous expression with non-uniform weights is straightforward):
\begin{align}
\log p(y|\Theta) & = \log \int \int p(y,f,u|X,W,\Theta) dfdu \nonumber \\
& = \log  \int \int \left(  \prod_{a=1}^n \frac{1}{\sqrt{2\pi N_a\tau^a}}\exp\left(-\frac{(y^a-\sum_{i=1}^{N_a} f(x^a_i))^2}{2N_a\tau^a}\right) \right)  \frac{p(u|W)}{q(u)} p(f| u) q(u)dfdu \nonumber \\
& \geq  \int \int \log\left\{ \prod_{a=1}^n \frac{1}{\sqrt{2\pi N_a\tau^a}}\exp\left(-\frac{(y^a-\sum_{i=1}^{N_a} f(x_i^a))^2}{2N_a\tau^a}\right)  \frac{p(u|W)}{q(u)}\right\} p(f| u) q(u)dfdu \nonumber \\
& = -\frac{1}{2}\sum_a \int \int \left\{ \frac{(y^a)^2 - 2y^a \sum_{i=1}^{N_a}f(x^a_i) + \left(\sum_{i=1}^{N_a}f(x^a_i)\right)^2}{N_a\tau^a} \right\}  
p(f| u) q(u)dfdu \nonumber \\
& \quad - \frac{1}{2}\sum_a \log(2\pi N_a\tau^a) - \int q(u)\log \frac{q(u)}{p(u|W)} du.
\label{eq:VB_obj_normal}
\end{align}
Using again a Gaussian distribution for $q(u)$, we have $q(f)=\int p(f|u)q(u)du$, which is a normal distribution and let  $q^a(f^a)$ be its marginal normal distribution of $f^a=(f(x^a_1),\ldots,f(x^a_{N_a}))$ with mean and covariance given by $m^a$ and $S^a$ as before in (\ref{eqn:meanCov}).

Then all expectations with respect to $q(f)$ are tractable and the ELBO is simply 

\begin{align}
\mathcal{L}(q,\theta)&=- \frac{1}{2}\sum_{a=1}^n \left\{ \frac{(y^a)^2 - 2y^a {\bf 1}^\top m^a  + {\bf 1}^\top \left(S^a+m^a (m^a)^\top\right){\bf 1}}{N_a\tau^a}  \right\} - \frac{1}{2}\sum_a \log(2\pi N_a\tau^a) 
 \nonumber \\
{}& \qquad\qquad - KL(q(u)||p(u|W)).
\end{align}

\subsection{Exponential bag model}
In this case, we have $F(\eta)=-1/\eta$.
We can apply the similar argument as in Lemma \ref{lem:logsum}.  For any $\alpha_i> 0$ with $\sum_i\alpha_i=1$, by the concavity of $F$, 
\begin{align*}
 \int F\left(\sum_i w_i \Psi(v_i)\right)q(v_i) dv_i &=  \int  F\left(\sum_i \alpha_i w_i/ \alpha_i  \Psi(v_i)\right)q(v_i) dv_i  \\
  & \geq  \int \sum_i \alpha_i F\left(w_i/\alpha_i \Psi(v_i)\right)q(v_i)dv_i \\
  & = \sum_i \alpha_i \int F\left(w_i / \alpha_i \Psi(v_i)\right)q(v_i) dv_i.  
\end{align*}
For $F(\eta)=-1/\eta$, the last line is equal to 
\[
\sum_i \frac{\alpha_i^2}{w_i} \int \frac{1}{\Psi(v_i)}q(v_i)dv_i.
\]
When using a normal $q$, this is tractable for several choices of $\Psi$ including $e^v$ and $v^2$.  If we let $\xi_i:=  
 \int \frac{1}{\Psi(v_i)}q(v_i)dv_i$, and maximize 
 \[
 \sum_i \alpha_i^2\frac{\xi_i}{w_i}
 \]
 under the constraint $\sum_i \alpha_i = 1$, we obtain
 \[
 	\alpha_i = \frac{(w_i/\xi_i)}{\sum_\ell (w_i/\xi_i)}. 
\]
Finally, we have a lower bound 
\begin{equation}
\int F\left(\sum_i w^i \Psi(v_i)\right)q(v_i) dv_i \geq -\frac{\sum_i (w_i/\xi_i)}{\sum_i (w_i/\xi_i)^2}
\end{equation}
where 
\[
\xi_i= \int \frac{1}{\Psi(v_i)}q(v_i)dv_i.
\]
which is tractable for a Gaussian variational family. Also with an explicit form of $\Psi$, it is easy to take the derivatives of the resulting lower bound with respect to the variational parameters in $q(v)$.

\section{Alternative approaches}
\label{sec:baselines}
\paragraph{Constant}
For the Poisson model, we can take $\lambda^a_i = \lambda^a_c$, a constant rate across the bag, then:
\begin{equation*}
\hat \lambda^a_c = \frac{y^a}{p^a} 
\end{equation*}
hence the individual level predictive distribution is the form $y^a_i \sim Poisson( \hat\lambda^a_c)$, and for unseen bag $r$, $\hat\lambda^\text{bag}_c =  \frac{1}{\sum_{a=1}^n p^a } \sum_{a=1}^n y^a$, with predictive distribution given by $y^r \sim Poisson( p^r \hat\lambda_{c}^{\text{bag}})$.

\paragraph{bag-pixel: Bag as Individual}
Another baseline is to train a model from the weighted average of the covariates, given by $x^a =\sum_{i=1}^{N_a} \frac{p^a_i}{p^a} x^a_i$ in the Poisson case, and $x^a =  \sum_{i=1}^{N_a} \frac{w^a_i}{w^a} x^a_i$ in the normal case. The purpose of this baseline is to demonstrate that modelling at the individual level is important during training. Since we now have labels and covariates at the bag level, we can consider the following model:
\begin{equation*}
y^a | x^a \sim Poisson(p^a \lambda(x^a))
\end{equation*}
with $\lambda(x^a) = \Psi(f(x^a))$ for the Poisson model. For the normal model, we have:
\begin{equation*}
y^a | x^a \sim Normal(w^a \mu(x^a), (w^a)^2 \tau)
\end{equation*}
where $\mu(x^a) = f(x^a)$ and $\tau$ is a parameter to be learnt (assuming constant across bags). Now we observe that these models are identical to the individual model, except for a difference in indexing. Hence, after learning the function $f$ at the bag level, we can transfer the model to the individual level. Essentially here we have created fake individual level instances by aggregation of individual covariates inside a bag.

\paragraph{Nystr\"om: Bayesian MAP for Poisson regression on explicit feature maps} 
Instead of the posterior based on the model \eqref{eq:likelihood}, we can also consider an explicit feature map in order to directly construct a MAP estimator. While this method does not provide posterior uncertainty over $\lambda^a_i$, it does provide an interesting connection to the settings we have considered and also manifold-regularized neural networks, as discussed below. Let $K_{zz}$ be the covariance function defined on covariates $\{ z_1, \dots z_n \}$, and consider its low rank approximation $K_{zz} \approx \mathbf{k}_{zW} K_{WW}^{-1} \mathbf{k}_{W\hspace{-0.03cm}z}$ with landmark points $W = \{w_\ell\}_{\ell=1}^m$ and ${\bf k}_{zW}=(k(z,w_1),\ldots,k(z,w_\ell))^T$. By using landmark points $W$, we have avoided computation of the full kernel matrix, reducing computational complexity. Under this setup, we have that $K_{zz} \approx \Phi_z {\Phi^\top_z}$, with $\Phi_z=\mathbf{k}_{zW} K_{WW}^{-\frac{1}{2}}$ being the explicit (Nystr\"om) feature map. Using this explicit feature map $\Phi$, we have the following model: 
\begin{gather*}
f^a_i = \phi^a_i \beta,  \qquad  \beta  \sim  \mathcal{N}(0, \gamma^2 I ) \\
y^a | \xa   \sim {\rm Poisson}\left( \sum_{i=1}^{N_a} p^a_i \lambda(x^a_i) \right),\qquad \lambda(x^a_i)  = \Psi(f^a_i), 
\end{gather*}
where $\gamma$ is a prior parameter and $\phi^a_i$ is the corresponding $i^{th}$ row of $\Phi_{\xa}$. We can then consider a MAP estimator of the model coefficients $\beta$:
\begin{equation}
\hat\beta = \text{argmax}_\beta \log [\Pi_{a=1}^n p(y^a | \beta, \xa)] + \log p(\beta).
\end{equation}
This essentially recovers the same model as in (\ref{eqn:poiss_nll}) with the standard $l_2$ loss regularising the complexity of the function. This model can be thought of in several different ways, for example as a weight space view of the GP (\cite{williams2006gaussian} for an overview), or as a MAP of the Subset of Regressors (SoR) approximation  \cite{smola2001sparse} of the GP when $\sigma = 1$. Additional we may include manifold regulariser as part of the prior, see discussion below about neural network.
\paragraph{NN: Manifold-regularized neural networks}
\label{app:manifold_nn}
The next approach we consider is a parametric model for $f$ as in \cite{kotzias2015group}, and search the best parameter to minimize negative log-likelihood $\ell_0$ across bags.  This paper considers a neural network with parameters $\theta$ for the model $f$, and uses the back-propagation to learn $\theta$ and hence individual level model $f$. However, since we only have aggregated observations at the bag level, but lots of individual covariate information, it is useful to incorporate this information also, by enforcing smoothness on the data manifold given by the unlabelled data. To do this, following \cite{kotzias2015group} and \cite{patrini2014almost}, we pursue a semisupervised view of the problem and include an additional manifold regularisation term \cite{belkin2006manifold} (rescaling with $N_{\text{total}}^2$ during implementation): 
\begin{equation}
\ell_1 = \sum_{w=1}^{N_{\text{total}}} \sum_{u=1}^{N_{\text{total}}}  (f_u-f_w)^2 k_L(x_u, x_w) = \mathrm{f}^\top \mathrm{L} \ \mathrm{f}
\end{equation}
where we have suppressed the bag index, $N_{\text{total}}$ represents the total number of individuals, $k_L(\cdot, \cdot) $ is some user-specified kernel\footnote{In practice, this does not have to be a positive semi-definite kernel, it can be derived from any notion of similarity between observations, including k-nearest neighbours.}, $\mathrm{f} = [f_1, \dots, f_{N_{\text{total}}}]^\top$, $\mathrm{L}$ is the Laplacian defined as $\mathrm{L} = diag(\mathrm{K_{L}} \mathbbm{1}^\top) - \mathrm{K_{L}}$, where $\mathbbm{1}$ is just $[1,\dots,1]$ and $ \mathrm{K_{L}}$ is a kernel matrix.  Although this term involves calculation of a kernel matrix across individuals, in practice we consider stochastic gradient descent (SGD) and also random Fourier features \cite{rahimi2007random} or Nystr\"{o}m approximation (see Appendix \ref{app:rff}), with scale parameter $\lambda_1$ to control the strength of the regularisation. Similarly, one can also consider manifold regularisation at the bag level, if bag-level covariates/embeddings are available, for further details, see Appendix \ref{app:bag_reg}. 

In fact, the same regularisation can be applied to the MAP estimation with the explicit feature maps. This is equivalent to having a prior $\beta \sim \mathcal{N}(0, \sigma^2 I + ( \lambda_1 \Phi^\top \mathrm{L} \Phi)^{-1} )$ that is data dependent and incorporates the structure of the manifold \footnote{In order to guarantee positive definiteness of Laplacian, one can add $\epsilon I$, where $\epsilon > 0$.}. 

For implementation, we consider a one hidden layer neural network, with also an output layer, for a fair comparison to the Nystr\"om approach. For activation function, we consider the Rectified Linear Unit (ReLU).

\paragraph{MAP estimation of GP}
We introduce $p(f, u) = p(f| u)p(u|W)$ and consider the posterior given by $p(u | f, y, w, \theta)$, where here the conditional distribution $f | u$ is given by:
\begin{equation}
f|u\sim  GP(\tilde{\mu}_u,\tilde{K}),
\end{equation}
\[
	\tilde{\mu}(z)= \mu_z + {\bf k}_{zW}K_{WW}^{-1} (u - \mu_W), \quad \tilde{K}(z,z')=k(z,z')-{\bf k}_{zW}K_{WW}^{-1} {\bf k_{Wz'}}
\]
where ${\bf k}_{zW}=(k(z,W_1),\ldots,k(z,W_\ell))^T$. Using Bayes rule, we obtain:
\begin{eqnarray*}
\log [p(u | f, y, w)] & = & \log[ p(y | f, u)  p(f, u | X, W)] \\
& = & \log [p(y|f) p(f|u, X) p(u|W)] \\
& = & \sum_{a=1}^n y^a \log(p^a \lambda^a) + \sum_{a=1}^n p^a\lambda^a - \sum_{a=1}^n \log( y^a !) + \log ( p(f|u, X)) + \log( p(u | W) )
\end{eqnarray*}
where $p(f|u, X) \sim \mathcal{N}(\tilde{\mu}_u, \tilde{K})$ given by above, and $p(u|W) \sim \mathcal{N}(\mu_W, \Sigma_{WW})$, i.e. 
\begin{equation}
\log p(f|u, X) + \log p(u | W) = -\frac{1}{2} ( \log (|\tilde{K}| |\Sigma_{WW}|)+ (f - \tilde{\mu}_u)^\top \tilde{K}^{-1} (f - \tilde{\mu}_u) + (u - \mu_W)^\top \Sigma_{WW}^{-1} (u - \mu_W)
\end{equation}
Here, we can not perform SGD, as the latter terms does not decompose into a sum over the data. More importantly, here we require the computation of $\tilde{K}$, which contains the kernel matrix $K$, even after the use of landmarks. This direct approach is not feasible for large number of individuals, which is true in our target application, and hence we do not pursue this method, and consider Nystr\"om and NN as baselines.
\section{Random Fourier Features on Laplacian}
\label{app:rff}
Here we discuss using random Fourier features \cite{rahimi2007random} to reduce computational cost in calculation of the Laplacian defined as $\mathrm{L} = diag(\mathrm{K} \mathbbm{1}^\top) - \mathrm{K}$, where $\mathbbm{1}$ is just $[1,\dots,1]$ and $ \mathrm{K}$. Suppose the kernel is stationary i.e. $k_w(x-y) = k(x, y)$ (some examples include the gaussian and matern kernel), then using random Fourier features, we obtain $\mathrm{K} \approx \Phi \Phi^\top$, where $\Phi \in \mathbb{R}^{ b_N \times m}$, $b_N$ denotes the total number of individuals in the batch and $m$ denotes the number of frequencies. Now we have:
\begin{equation}
\mathrm{f}^\top \mathrm{L} \mathrm{f} \approx \mathrm{f}^\top diag(\Phi \Phi^\top \mathbbm{1}^\top) \mathrm{f} -  \mathrm{f}^\top \Phi \Phi^\top \mathrm{f} = \mathrm{f}^\top diag(\mathrm{\Phi \Phi^\top} \mathbbm{1}^\top) \mathrm{f} - || \Phi^\top f ||^2_2
\end{equation}
In both terms, we can avoid computing the kernel matrix, by carefully selecting the order of computation. Note another option is to consider Nystr\"{o}m approximation with landmark points $\{ z_1, \dots z_m\}$, then $K \approx K_{nm} K_{mm}^{-1} K_{mn}$, where $K_{mm}$ denotes the  kernel matrix on landmark points, while $K_{nm}$ is the kernel matrix between landmark and data. Then $\Phi = K_{nm}  K_{mm}^{-\frac{1}{2}}$. 
\section{Bag Manifold regularisation}
\label{app:bag_reg}
Suppose we have bag covariates $s^a$ (note these are for the entire bag), and also some summary statistics of a bag, e.g. mean embeddings \cite{muandet2016kernel} given by $H^a = \frac{1}{N_a}\sum_{i=1}^{N_a} h(x^a_i)$, with some user-defined $h$. Then similarly to individual level manifold regularisation, we can consider manifold regularisation at the bag level (assuming a seperable kernel for simplicity), i.e. 
\begin{equation}
\ell_2 = \sum_{l=1}^n \sum_{m=1}^n (F^l - F^m)^2 k_s(s^l, s^m) k_h(H^l, H^m) = \mathrm{F}^\top \mathrm{L}_{\text{bag}} \mathrm{F}
\end{equation}
where $F^a = \frac{1}{N_l} \sum_{i=1}^{N_a} f^a_i$, $k_s$ is a kernel on bag covariates $s^a$, $k_\mu$ is a kernel on $H^a$, $\mathrm{L}_{\text{bag}}$ is the bag level Laplacian with the corresponding kernel, and $F = [F^1, \dots, F^n]^\top$. Combining all these terms, we have the following loss function to minimise:
\begin{equation}
\ell = \frac{1}{b} \ell_0  + \frac{\lambda_1} {b_{N}^2} \ell_1 + \frac{\lambda_2}{b_{N}^2} \ell_2 
\end{equation}
where $b$ is the mini-batch size in SGD, $B_N$ is the total number of individuals in each mini-batch, $\lambda_1$ and $\lambda_2$ are parameters controlling the strength of the respective regularisation.
\section{Additional details for Poisson variational derivation}
\label{app:poisson_details}
\subsection{Log-sum lemma}
\label{app:lemma}
\begin{lem}\label{lem:lb}
Let $v=[v_{1},\ldots,v_{N}]^{\top}$ be a random vector with probability
density $q(v)$, and let $w_{i}\geq0$, $i=1,\ldots,N$. Then, for
any non-negative valued function $\Psi(v)$,
\[
\int\log\bigl(\sum_{i=1}^{N}w_{i}\Psi(v_{i})\bigr)q(v)dv\geq\log\Bigl(\sum_{i=1}^{N}w_{i}e^{\xi_{i}}\Bigr),
\]
where 
\[
\xi_{i}:=\int\log\Psi(v_{i})q_{i}(v_{i})dv_{i}.
\]
\end{lem}

\begin{proof}
Let $\alpha_{1},\ldots,\alpha_{N}$ be non-negative numbers with $\sum_{i=1}^{N}\alpha_{i}=1$.
It follows from Jensen's inequality that 
\begin{align}
\int\log\bigl(\sum_{i=1}^{N}w_{i}\Psi(v_{i})\bigr)q(v)dv & =\nonumber \\
\int\log\Bigl(\sum_{i=1}^{N}\alpha_{i}\tfrac{w_{i}}{\alpha_{i}}\Psi(v_{i})\Bigr)q(v)dv & \geq\nonumber \\
\sum_{i=1}^{N}\alpha_{i}\left[\int\log\Bigl(\Psi(v_{i})\Bigr)q(v_{i})dv_{i}+\log\frac{w_{i}}{\alpha_{i}}\right] & =\nonumber \\
\sum_{i=1}^{N}\alpha_{i}\xi_{i}+\sum_{i=1}^{N}\alpha_{i}\log\frac{w_{i}}{\alpha_{i}}.\label{eq: lma_proof}
\end{align}
By Lagrange multiplier method, maximizing the last line with respect
to $\alpha$ gives 
\[
\alpha_{i}=\frac{w_{i}e^{\xi_{i}}}{\sum_{j=1}^{N}w_{j}e^{\xi_{j}}}.
\]
Plugging this to (\ref{eq: lma_proof}) completes the proof.
\end{proof}
\subsection{A lower bound of marginal likelihood for $\Psi(f)=e^f$ and $\Psi(f)=f^2$}
\label{sec:lb_square}
Using Lemma \ref{lem:lb},  we obtain that 
\begin{equation}
\int \log\bigl(\sum_{i=1}^{N} p^a_i \Psi(v^a_i)\bigr) q(v^a)dv^a
\geq  \log\Bigl(\sum_{i=1}^{N} p^a_i \Psi(\xi^a_i) \Bigr),
\end{equation}
where 
\[
\xi^a_i = \int \log \Psi(v^a_i)q^a_i(v^a_i)dv^a_i.
\]
The above lower bound is tractable for the popular functions $\Psi(v)=v^{2}$
and $\Psi(v)=e^{v}$ under the normal variational distributions $q^a(v^a) \sim \mathcal{N}\left(m^a,S^a\right).$
In particular,
\begin{eqnarray*}
\Psi(v)=e^{v}: & \xi^{a}_i= & \int v^a_{i}q^a_{i}(v^a_{i})dv^a_{i}=m^a_{i},\\
\Psi(v)=v^{2}: & \xi^{a}_i= & \int\log (v^a_{i})^{2}q^a_{i}(v^a_{i})dv^a_{i}=-G\left(-\frac{m^a_i}{2S^a_{ii}}\right)+\log\left(\frac{S^a_{ii}}{2}\right)-\gamma,
\end{eqnarray*}
where $\gamma$ is the Euler constant and 
\[
G(t)=2t\sum_{j=0}^{\infty}\frac{j!}{\left(2\right)_{j}\left(3/2\right)_{j}}t^{j}
\]
 is the partial derivative of the confluent hypergeometric function \cite{lloyd2015variational,ancarani2008derivatives}. However, in this work we focus on the Taylor series approximation for $\Psi(v)=v^{2}$, as implementation of the above bound uses a large look-up table and involves linear interpolation. Furthermore, it is suggested in experiments that the secondary lower bound proposed above in Lemma \ref{lem:lb} can lead to poor calibration, for more details, refer to Section \ref{sec:exp}.
\subsection{KL Term}
\label{app:kl_term}
Since $q(u)$ and $p(u|W)$ are both normal distribution, the KL divergence is tractable:
\begin{equation}
KL(q(u)||p(u|W)) = \frac{1}{2}\Bigl\{ Tr[K_{WW}^{-1}\Sigma_u] + \log\frac{|K_{WW}|}{|\Sigma_u|} - m + (\mu_W-\eta_u)^T K_{WW}^{-1} (\mu_W-\eta_u) \Bigr\}
\end{equation}
\subsection{Taylor series approximation in the variational method}
\label{sec:taylor}
We consider the integral
$$\int \log\bigl(\sum_{i=1}^{N} p^a_i (v^a_i)^2 \bigr) q^a(v^a)dv^a$$
where $q^a$ is $\mathcal N(m^a, S^a)$. We note that this can be written as $\mathbb{E}\log\left\Vert V^a\right\Vert ^{2}$, where $V^a \sim N( \tilde{m}^a, \tilde{S}^a)$,
with $P^a = diag\left(p^a_1,\dots, p^a_{N_a}\right), \tilde{m}^a = {P^a}^{1/2} m^a$ and $\tilde{S}^{a} = {P^a}^{1/2} S^a{P^a}^{1/2}$. Note that $\left\Vert V^a\right\Vert ^{2}$ follows a non-central chi-squared distribution. We now resort to a Taylor series approximation for $\mathbb{E}\log\left\Vert V^a\right\Vert ^{2}$ (similar to \cite{teh2007collapsed}) around $\mathbb{E}\left\Vert V^a\right\Vert ^{2}=\left\Vert \tilde{m}^a \right\Vert ^{2}+tr\tilde{S}^a$,
resulting in 
\begin{eqnarray*}
\mathbb{E}\log\left(\left\Vert V^a\right\Vert ^{2}\right) & = & \log\left(\mathbb{E}\left\Vert V^a\right\Vert ^{2}\right)\\
&&+\mathbb{E}\left[\frac{\left\Vert V^a\right\Vert ^{2}-\mathbb{E}\left\Vert V^a\right\Vert ^{2}}{\mathbb{E}\left\Vert V^a\right\Vert ^{2}}-\frac{\left(\left\Vert V^a\right\Vert ^{2}-\mathbb{E}\left\Vert V^a\right\Vert ^{2}\right)^{2}}{2\left(\mathbb{E}\left\Vert V^a\right\Vert ^{2}\right)^{2}}+\mathcal{O}\left(\left(\left\Vert V^a\right\Vert ^{2}-\mathbb{E}\left\Vert V^a\right\Vert ^{2}\right)^{3}\right)\right]\\
 & \approx & \log\left(\left\Vert \tilde{m}^a\right\Vert ^{2}+tr\tilde{S}^a\right)-\frac{2 \tilde{m}^{a\top}\tilde{S}^a\tilde{m}^a+tr\left(\left(\tilde{S}^a\right)^2\right)}{\left(\left\Vert \tilde{m}^a\right\Vert ^{2}+tr\tilde{S}^a\right)^{2}}.
\end{eqnarray*}

As commented in \cite{teh2007collapsed}, approximation is very accurate when $\mathbb{E}\left\Vert V^a\right\Vert ^{2}$ is large, but the caveat is that the Taylor series converges only for $\left\Vert V\right\Vert ^{2}\in(0,2\mathbb{E}\left\Vert V\right\Vert ^{2})$
so this approach effectively ignores the tail of the non-central chi-squared.

\section{Code}
\label{app:implementaion}
All of our models were implemented in TensorFlow, and code will be published and available for use. 

\section{Additional Malaria Experimental Results}
\label{app:malaria}
Here we provide additional experimental results for the malaria dataset. In table \ref{tab:malaria}, we provide results for bag level performance for NLL and MSE with $10$ different test sets (after retrial of the experiments, splitting the data across train, early-stop, validation and testing). Statistical significance was not establish for the best performing Nystr\"{o}m method versus the VBAgg methods, this is shown in Table \ref{tab:malaria_p_values}. We further provide additional prediction/uncertainty patches for $3$ different splits to highlight the general behaviour of the trained models, with further explanation and details below.  

It is also noted in all cases $\lambda^a_i$ is the incidence rate per $1000$ people. For VBAgg and Nystr\"{o}m, we use an additive kernel, between an ARD kernel and a Matern kernel:
\begin{equation}
\label{eqn:malaria_kernel}
k((x, s_x),(y,s_y)) = \gamma_{1} \exp\left(-\frac{1}{2}\sum_{k=1}^{18} \frac{1}{\ell_k}(x_k - y_k)^2 \right) + \gamma_{2} \left(1 + \frac{\sqrt{3}\vert\vert s_x - s_y \vert {\vert}_2 }{\rho}\right) \exp \left(- \frac{\sqrt{3}\vert\vert s_x - s_y \vert {\vert}_2 }{\rho} \right)
\end{equation}
where $x, y$ are covariates, and $s_x$, $s_y$ are their respective spatial location. Here, we learn any scale parameters and weights during training. For the NN, we also use this kernel as part of manifold regularisation, however we use an RBF kernel instead of an ARD kernel, due to parameter tuning reasons (we can no longer learn these scales).

For constant model, bag rate predictions are computed by, $p^a \hat\lambda_{c}^{\text{bag}}$,where $\hat\lambda^\text{bag}_c =  \frac{1}{\sum_{a=1}^n p^a } \sum_{a=1}^n y^a$. This essentially takes into account of population. 
\begin{table}[H]
\centering
\caption{Results for the Poisson Model on the malaria dataset with $10$ different re-splits of train, early-stopping, validation and test. Approximately, $191$ bags are used for test set.  Bag performance is measured on a test set, with MSE computed between $\log(y^a)$ and $\log(\sum_{i=1}^{N_a} p^a_i \hat{\lambda}^a_i)$. Brackets include standard deviation.}
\label{tab:malaria}
\begin{tabular}{lllll}
          & Bag NLL & Bag MSE (Log) \\ \hline
Constant   & 173.1 (31.2)     & 4.08 (0.13)  \\
Nystr\"{o}m-Exp & 88.1 (25.1)  & 1.31 (0.15)   \\
VBAgg-Sq-Obj  &  94.1  (34.0) & 1.21 (0.05)   \\
VBAgg-Exp-Obj & 97.2 (39.6)  & 1.04 (0.11)  \\
VBAgg-Sq      & 97.6 (39.0)   & 1.38 (0.18)   \\
VBAgg-Exp     & 99.2 (39.8)  & 1.21 (0.19)  \\
NN-Exp        & 164.4 (127.8) & 1.82 (0.29)   \\ 
\hline
\end{tabular}
\end{table}
\begin{table}[H]
\centering
\caption{p-values from a Wilcoxon signed-rank test for Nystr\"{o}m-Exp versus the methods below for Bag NLL and MSE for the malaria dataset. The null hypothesis is Nystr\"{o}m-Exp performs equal or worse than the considered method on the test bag performance.}
\label{tab:malaria_p_values}
\begin{tabular}{lll}
        & NLL           & MSE           \\ \hline
Constant &0.0009766& 0.0009766\\ 
NN-Exp  &  0.00293    & 0.0009766  \\
VBAgg-Sq-Obj & 0.1162 & 0.958 \\ 
VBAgg-Sq & 0.1377 & 0.1611 \\ 
VBAgg-Exp-Obj & 0.08008 & 1.0 \\ 
VBAgg-Exp  & 0.09668 & 0.958 \\ \hline
\end{tabular}
\end{table}
\begin{table}[H]
\centering
\caption{p-values from a Wilcoxon signed-rank test for VBAgg-Sq versus the methods below for Bag NLL and MSE for the malaria dataset. The null hypothesis is VBAgg-Sq performs equal or worse than the considered method on the test bag performance.}
\label{tab:p_malaria}
\begin{tabular}{lll}
        & NLL           & MSE           \\ \hline
Constant & 0.0009766 &0.0009766 \\ 
NN-Exp  &  0.01855    & 0.001953 \\
VBAgg-Sq-Obj & 0.6234 & 0.9861 \\ 
Nystr\"{o}m-Exp & 0.8838 & 0.8623  \\ 
VBAgg-Exp-Obj & 0.6875  & 1.0\\ 
VBAgg-Exp  & 0.3477 & 0.9346 \\ \hline
\end{tabular}
\end{table}
\FloatBarrier
\subsection{Predicted log malaria incidence rate for various models}
\label{app:malaria_preds}
\paragraph{Constant: Bag level observed incidences}
This is the baseline with $\hat{\lambda}^a_i$ being constant throughout the bag, as shown in Figure \ref{fig:constant_malaria}. For training, we only use $60\%$ of the data.
\begin{figure}[H]
\centering
\includegraphics[width=\linewidth]{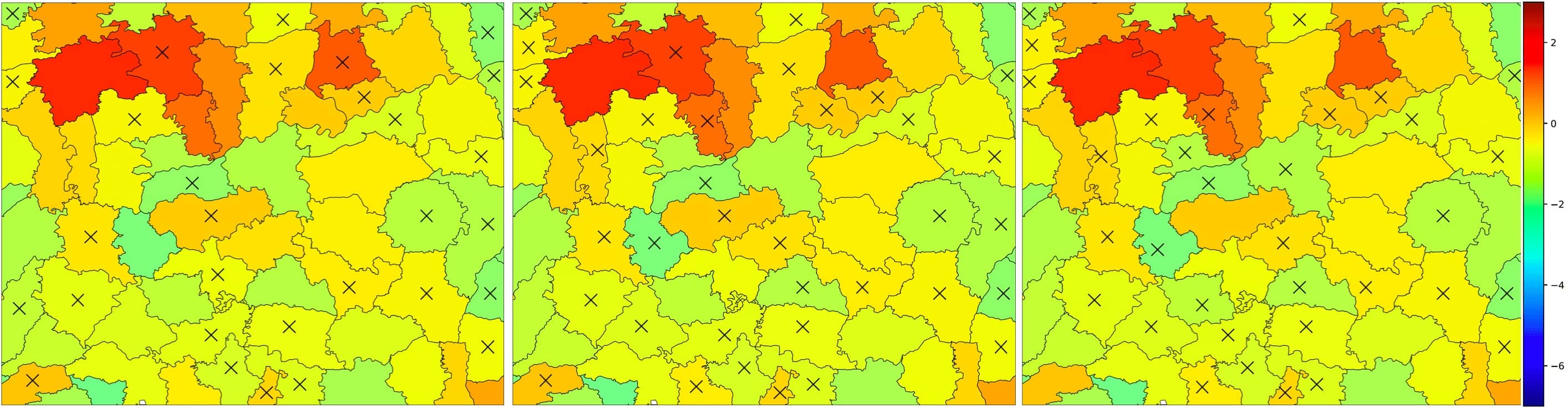}
\captionof{figure}{Predicted  $\hat{\lambda}^a_i$ on log scale using constant model, for $3$ different re-splits of the data.$
\times$ denote non-train set bags.}
\label{fig:constant_malaria}
\end{figure}
\FloatBarrier

\paragraph{VBAgg-Sq-Obj}
This is the VBAgg model with $\Psi(v) = v^2$ and tuning of hyperparameters is performed based on training objective, the lower bound to the marginal likelihood, we ignore early-stop and validation set here. The uncertainty of the model seems reasonable, and we also observe that in general the areas that are not in the training set have higher uncertainties. Furthermore, in all cases, malaria incidence was predicted to be higher near the river, as discussed in Section \ref{sec:malaria_exp}.
\begin{figure}[H]
\centering
\includegraphics[width=\linewidth]{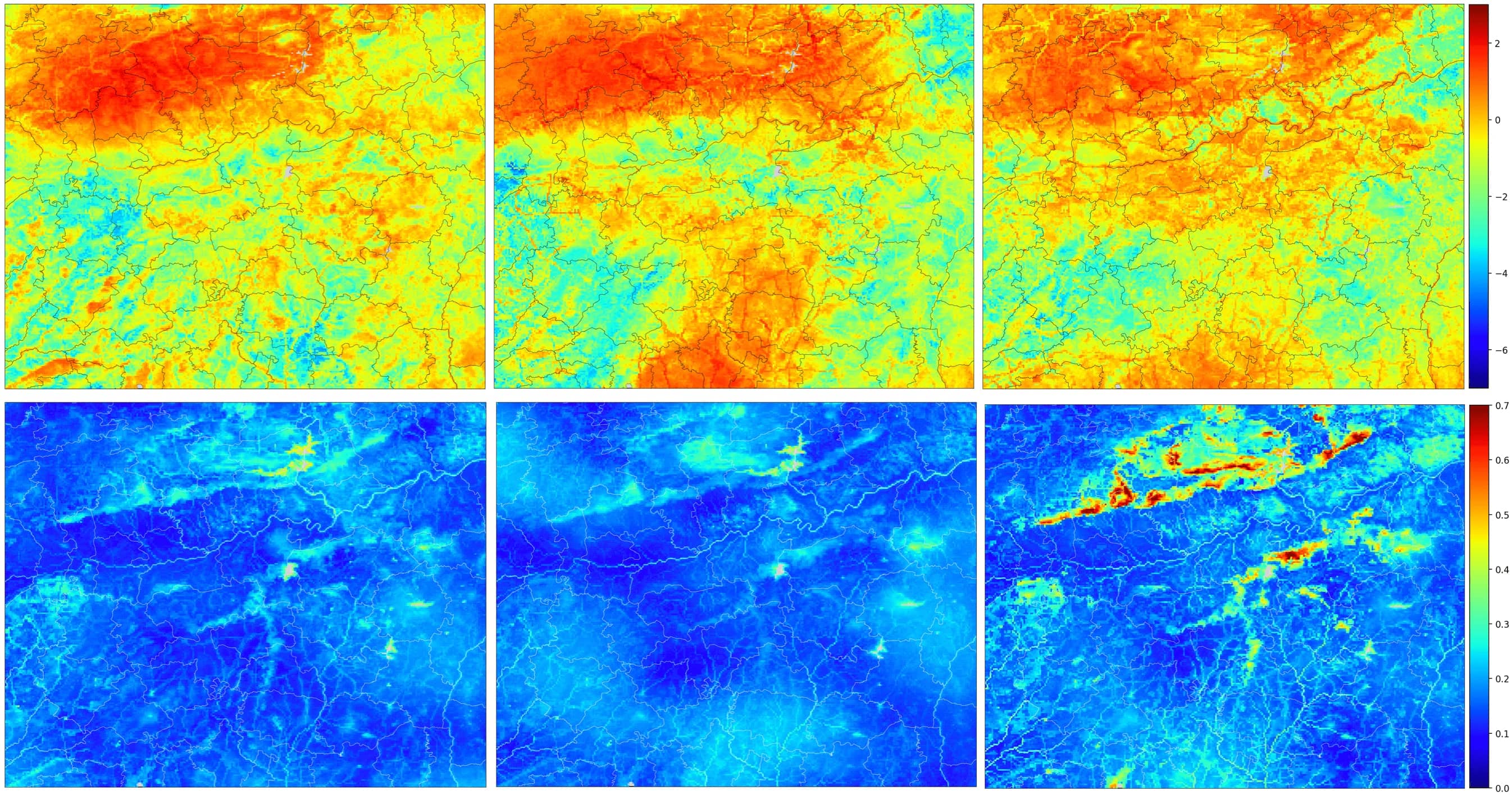}
\captionof{figure}{\textbf{Top:} Predicted  $\hat{\lambda}^a_i$ on log scale for VBAgg-Sq-Obj. \textbf{Bottom:} Standard deviation of the posterior $v$ in (\ref{eqn:meanCov}) with VBAgg-Sq-Obj.}
\label{fig:vbagg_square_obj_malaria}
\end{figure}
\FloatBarrier

\paragraph{VBAgg-Sq}
This is the VBAgg model with $\Psi(v) = v^2$ and tuning of hyperparameters is performed based on NLL at the bag level. Predicted incidence are similar to the VBAgg-Sq-Obj model. The uncertainty of the model is less reasonable here, this is expected behaviour, as we are tuning hyperparameters based on NLL here. In the first patch, the same parameters was chosen as VBAgg-Sq-Obj.
\begin{figure}[H]
\centering
\includegraphics[width=\linewidth]{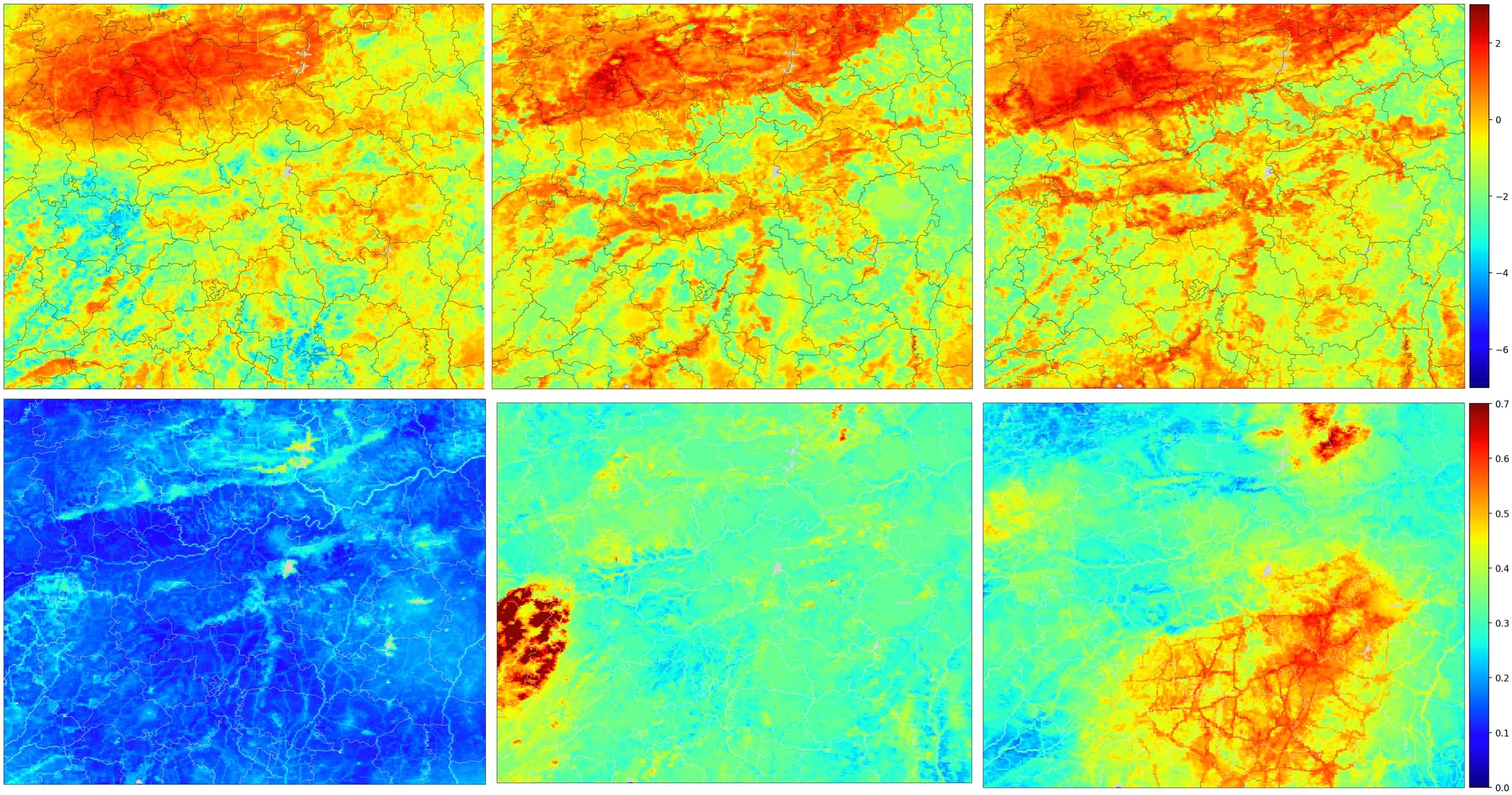}
\captionof{figure}{\textbf{Top:} Predicted  $\hat{\lambda}^a_i$ on log scale for VBAgg-Sq. \textbf{Bottom:} Standard deviation of the posterior $v$ in (\ref{eqn:meanCov}) with VBAgg-Sq.}
\label{fig:vbagg_square_malaria}
\end{figure}
\FloatBarrier

\paragraph{VBAgg-Exp-Obj}
This is the VBAgg model with $\Psi(v) = e^v$ and tuning of hyperparameters is performed based on training objective, the lower bound to the marginal likelihood, we ignore early-stop and validation set here. Predicted incidence seem to be stable in general, though some smoothness is observed. The uncertainty of the model is also not very reasonably here, but this behaviour was observed in the Toy experiments, and likely due to an additional lower bound.
\begin{figure}[H]
\centering
\includegraphics[width=\linewidth]{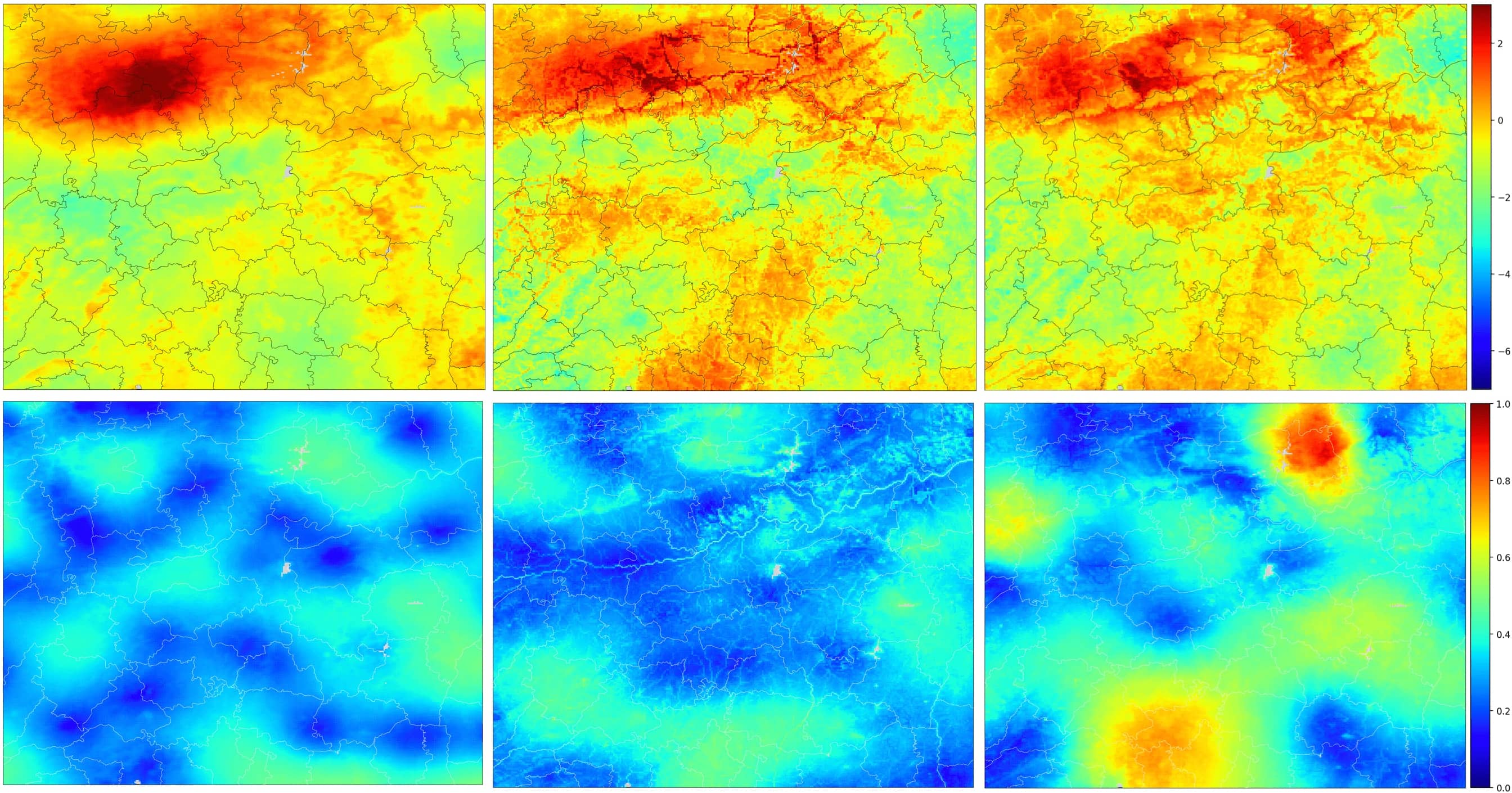}
\captionof{figure}{\textbf{Top:} Predicted  $\hat{\lambda}^a_i$ on log scale for VBAgg-Exp-Obj.\textbf{Bottom:} Standard deviation of the posterior $v$ in (\ref{eqn:meanCov}) with VBAgg-Exp-Obj.}
\label{fig:vbagg_exp_obj_malaria}
\end{figure}
\FloatBarrier

\paragraph{VBAgg-Exp}
This is the VBAgg model with $\Psi(v) = e^v$ and tuning of hyperparameters is performed based on NLL. For details, see discussion above for the VBAgg-Exp-Obj model.
\begin{figure}[H]
\centering
\includegraphics[width=\linewidth]{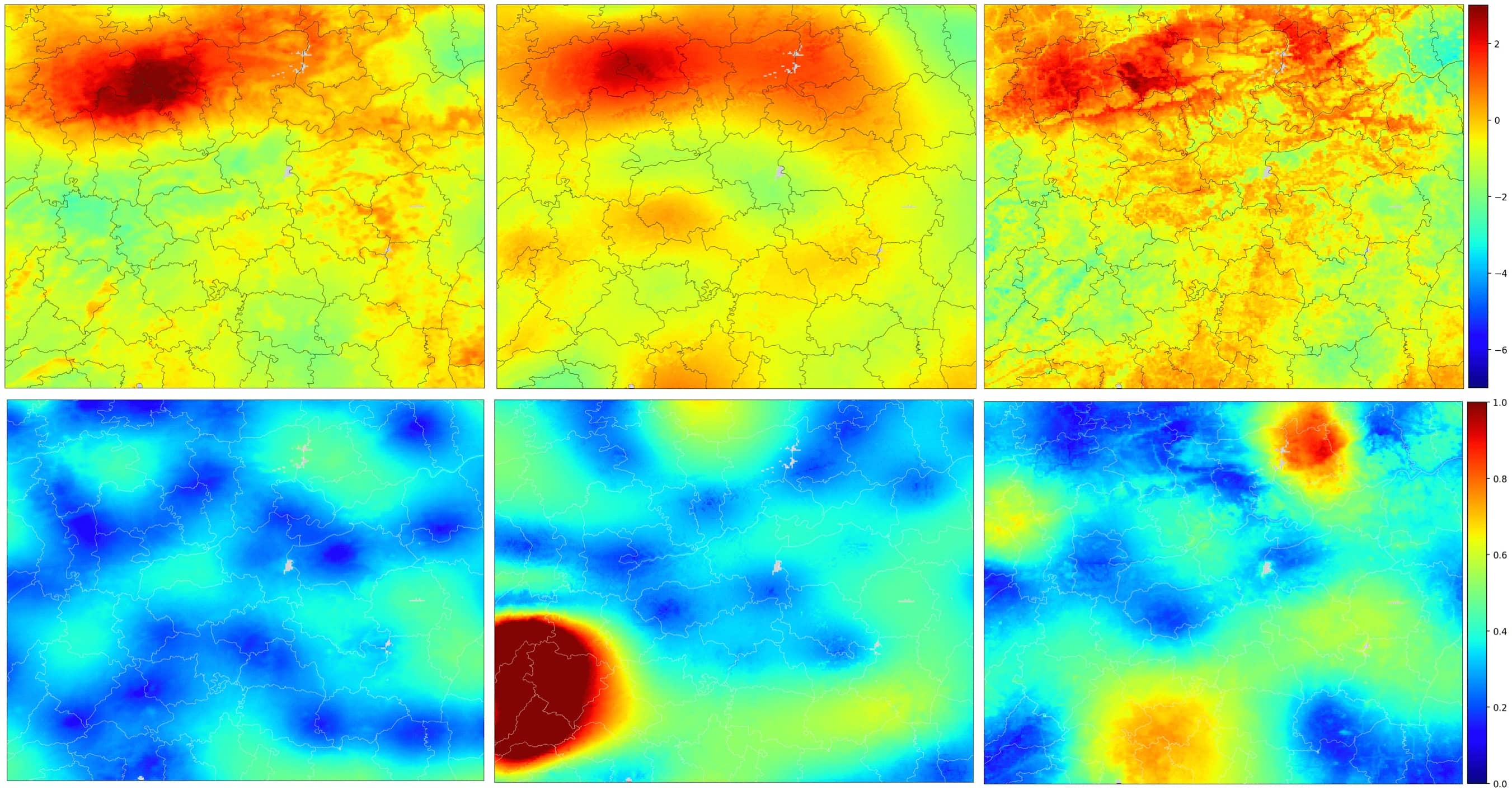}
\captionof{figure}{\textbf{Top:} Predicted  $\hat{\lambda}^a_i$ on log scale for VBAgg-Exp. \textbf{Bottom:} Standard deviation of the posterior $v$ in (\ref{eqn:meanCov}) with VBAgg-Exp.}
\label{fig:vbagg_exp_malaria}
\end{figure}
\FloatBarrier

\paragraph{Nystr\"{o}m-Exp}
This is the Nystr\"{o}m-Exp model, it is clear that while it performs best in terms of bag NLL, sometimes prediction are too smooth in the pixel space, this is because it optimises directly bag NLL. This pattern might be seen to be unrealistic, and may cause useful covariates to be neglected.
\begin{figure}[H]
\centering
\includegraphics[width=\linewidth]{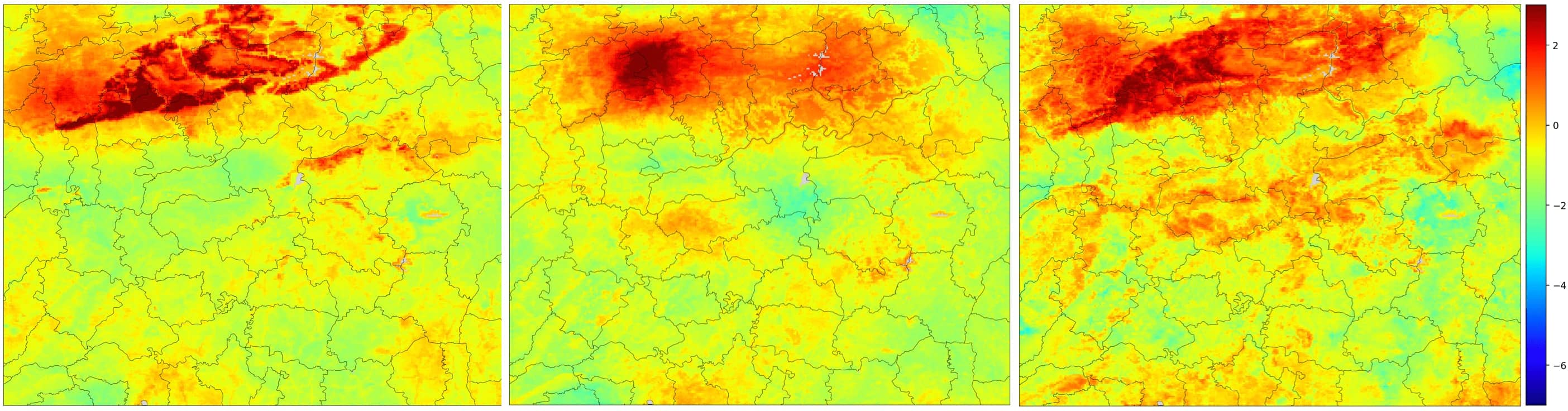}
\captionof{figure}{Predicted  $\hat{\lambda}^a_i$ on log scale for Nystr\"{o}m-Exp.}
\label{fig:gpmap_malaria}
\end{figure}
\FloatBarrier

\paragraph{NN-Exp}
We can see that the model is not very stable, this can be potentially due to the model does not have an inbuilt spatial smoothness function unlike other methods. It only uses manifold regularisation for training. Also, the maximum predicted pixel level intensity rate $\hat{\lambda}^a_i$ is over $1000$ in some cases, this is clearly physically impossible given $\lambda^a_i$ is rate per $1000$ people.
\begin{figure}[H]
\centering
\includegraphics[width=\linewidth]{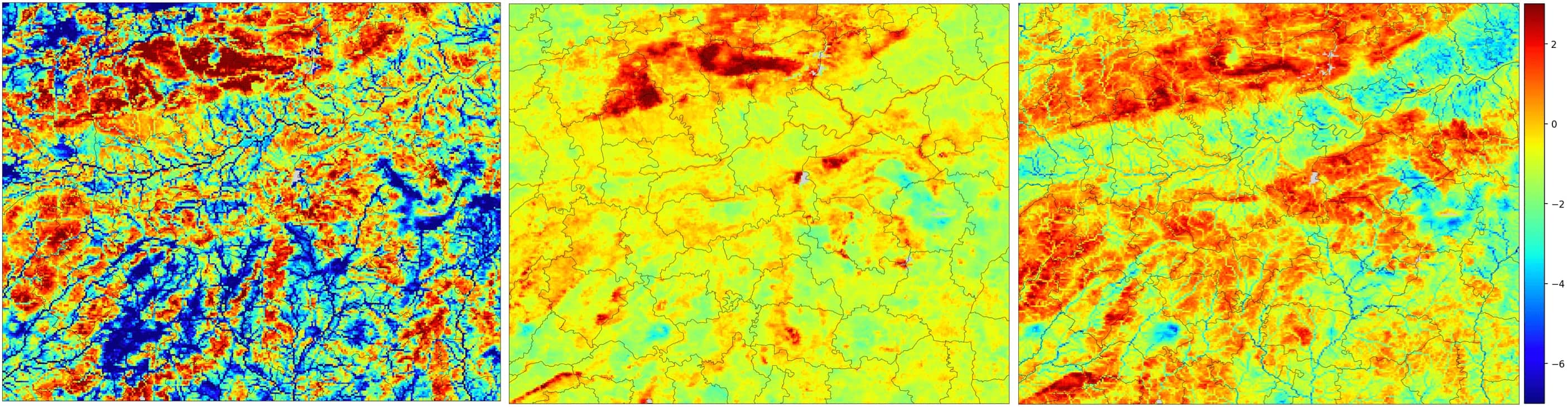}
\captionof{figure}{Predicted  $\hat{\lambda}^a_i$ on log scale for NN-Exp.}
\label{fig:nn_malaria}
\end{figure}

\subsection{Remote Sensing covariates that provide the existence of a river}
\FloatBarrier
\label{app:river}
Here, we provide figures for some covariates that give information that there is a river as indicated by the triangles in Figure \ref{fig:malaria}. 
\begin{figure}[H]
\centering
\includegraphics[width=0.7\linewidth]{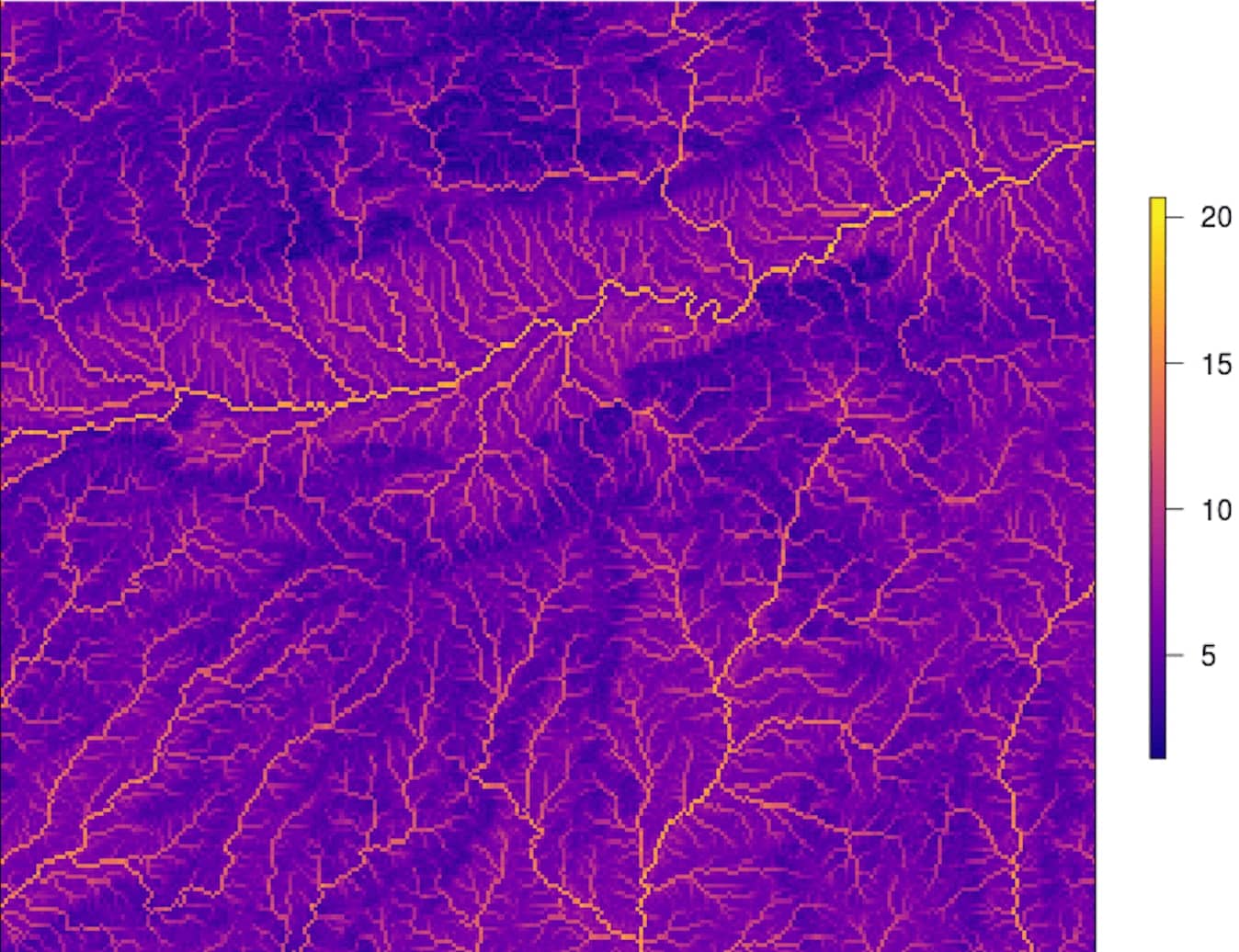}
\captionof{figure}{Topographic wetness index, measures the wetness of an area, rivers are wetter than others, as clearly highlighted.}
\label{fig:wetness}
\end{figure}
\begin{figure}[H]
\centering
\includegraphics[width=0.7\linewidth]{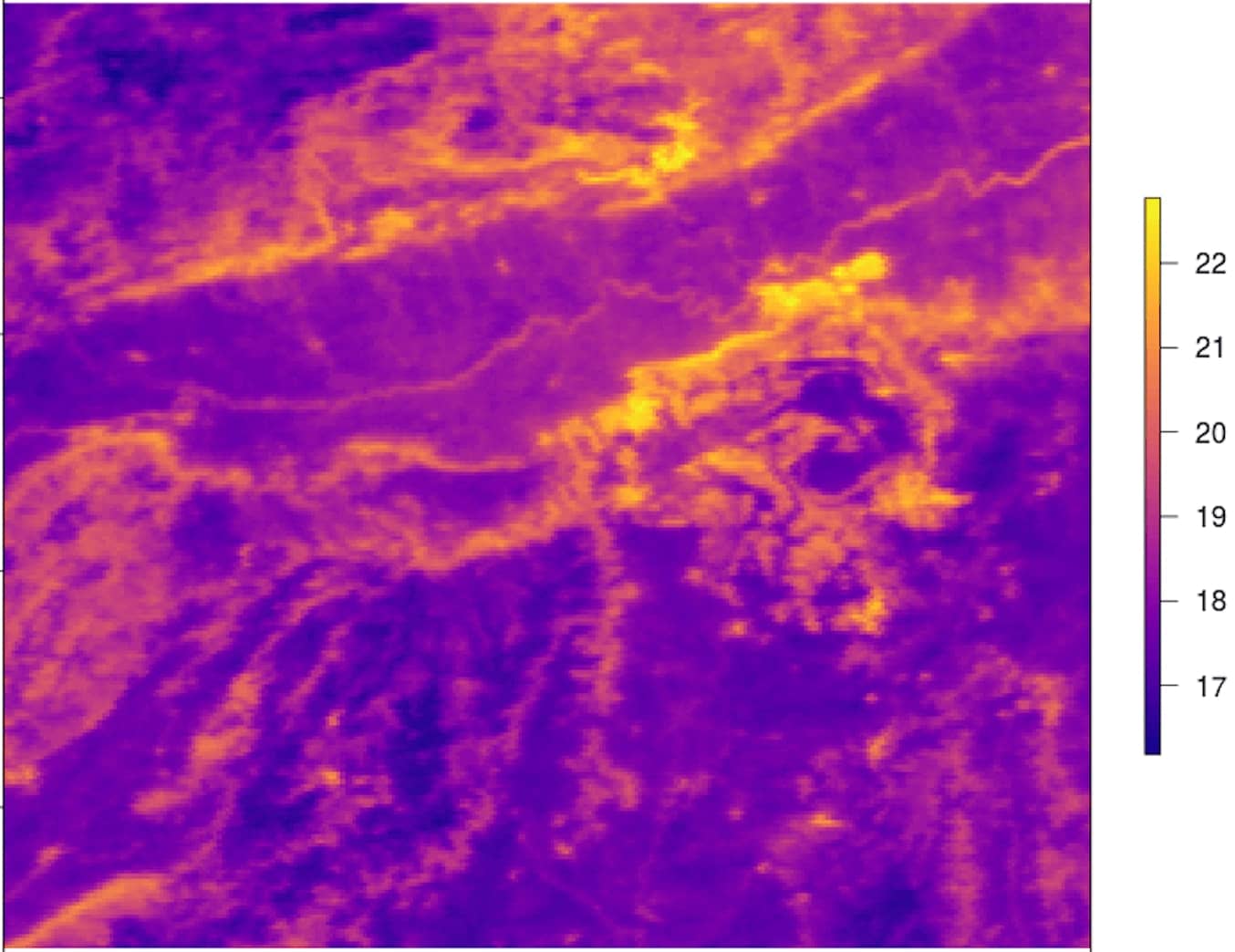}
\captionof{figure}{Land Surface Temperature at night, river is hotter at night, due to river being able to retain heat better.}
\label{fig:lst}
\end{figure}

\section{Additional Toy Experimental Results}
\label{app:experiments}
In this section, we provide additional experimental results for the Normal and Poisson model. In particular, we provide results on test bag level performance, and provide also prediction, calibration and uncertainty plots. 

For the VBAgg model, during the tuning process, it is possible to choose tuning parameters (e.g. learning rate, multiple-initialisations, landmark choices) based on NLL with an additional validation set or on the objective $\mathcal{L}_1$ on the training set. To compare the difference, we denote the model tuned on NLL as VBAgg and the model tuned on $\mathcal{L}_1$ as VBAgg-Obj. 
Intuitively, as VBAgg-Obj attempts to obtain as tight a bound to the marginal likelihood, we would expect better performance in calibration, i.e. more accurate uncertainties.

For calibration plots, we compute the $\alpha$ quantiles of the approximated posterior distribution and consider the ratio of times the underlying rate parameter $\lambda^a_i$ (or $\mu^a_i$ for the normal model) appear inside the quantiles of the posterior distribution. If the model provides good uncertainties/calibration, we should expect to see the quantiles to match with the observed ratio.

In the case of $\Psi(v)=v^2$, the approximated posterior distribution is simply a non-central $\chi^2$ distribution, while for $\Psi(v)=e^v$, this is a log-normal distribution. For the Normal Model, it is simply a normal distribution, as we do not have any transformations. Calibration plots can be found in Figure \ref{fig:normal_calibrate_varybag} and Figure \ref{fig:normal_calibrate_varyindiv} 
for the Normal Model, with Figure \ref{fig:poisson_calibrate_varybag} and Figure \ref{fig:poisson_calibrate_varyindiv} for the Poisson Model. 

For uncertainty plots, we plot the standard deviation of the posterior of $v \sim \mathcal{N}(m^a, S^a)$ (i.e. before transformation through $\Psi$), as this provides better interpretability. Uncertainty plots can be found in Figure \ref{fig:poisson_uncer} and \ref{fig:normal_uncer}

To demonstrate statistical significance of our result, we aggregate the repetitions in each experiment for each method and consider a one sided rank permutation test (Wilcoxon signed-rank test) to see whether VBAgg is statistically significant better than other approaches for individual NLL and MSE.

\subsection{Poisson Model}
\label{app:poisson}
\subsubsection{Swiss Roll Dataset}
\label{app:poisson_toy}
We provide additional results here for the experimental settings that we consider.

The varying number of bags experimental results is found in Figure \ref{fig:poisson_varybag}, with the corresponding table of p-values in Table \ref{tab:p_varybag_poisson_sq}, \ref{tab:p_varybag_poisson_exp} demonstrating statistical significance of the VBAgg-Exp and VBAgg-Sq method. Similarly, the varying number of individuals per bag through $N_{mean}$ experimental result can be found in Figure \ref{fig:poisson_varyindiv}, with the corresponding table of p-values in Table \ref{tab:p_varyindiv_poisson_sq}, \ref{tab:p_varyindiv_poisson_exp}. The comparison between VBAgg-Exp and VBAgg-Sq was found to be non-significant. 
\begin{figure}
\includegraphics[width=\linewidth]{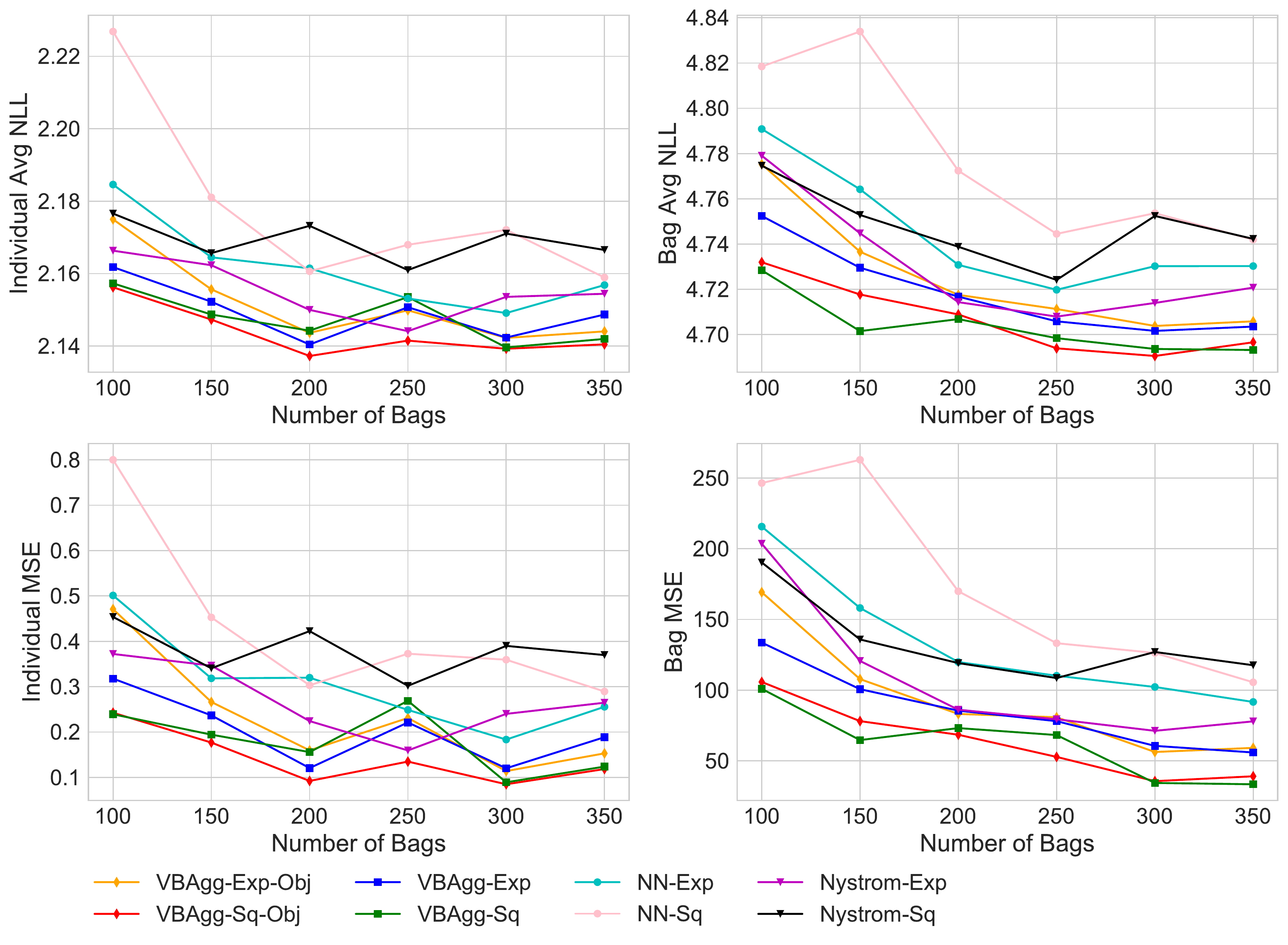}
\caption{Varying number of bags over $5$ repetitions.\textbf{Left Column:} Individual average NLL and MSE on train set. \textbf{Right Column:} Bag average NLL and MSE on test set (of size $500$). Constant prediction NLL and MSE is $2.23$ and $0.85$ respectively. bag-pixel model prediction NLL is above $2.4$ and MSE is above $3.0$, hence not shown on graph.}
\label{fig:poisson_varybag}
\end{figure}
\FloatBarrier
\begin{figure}
\includegraphics[width=\linewidth]{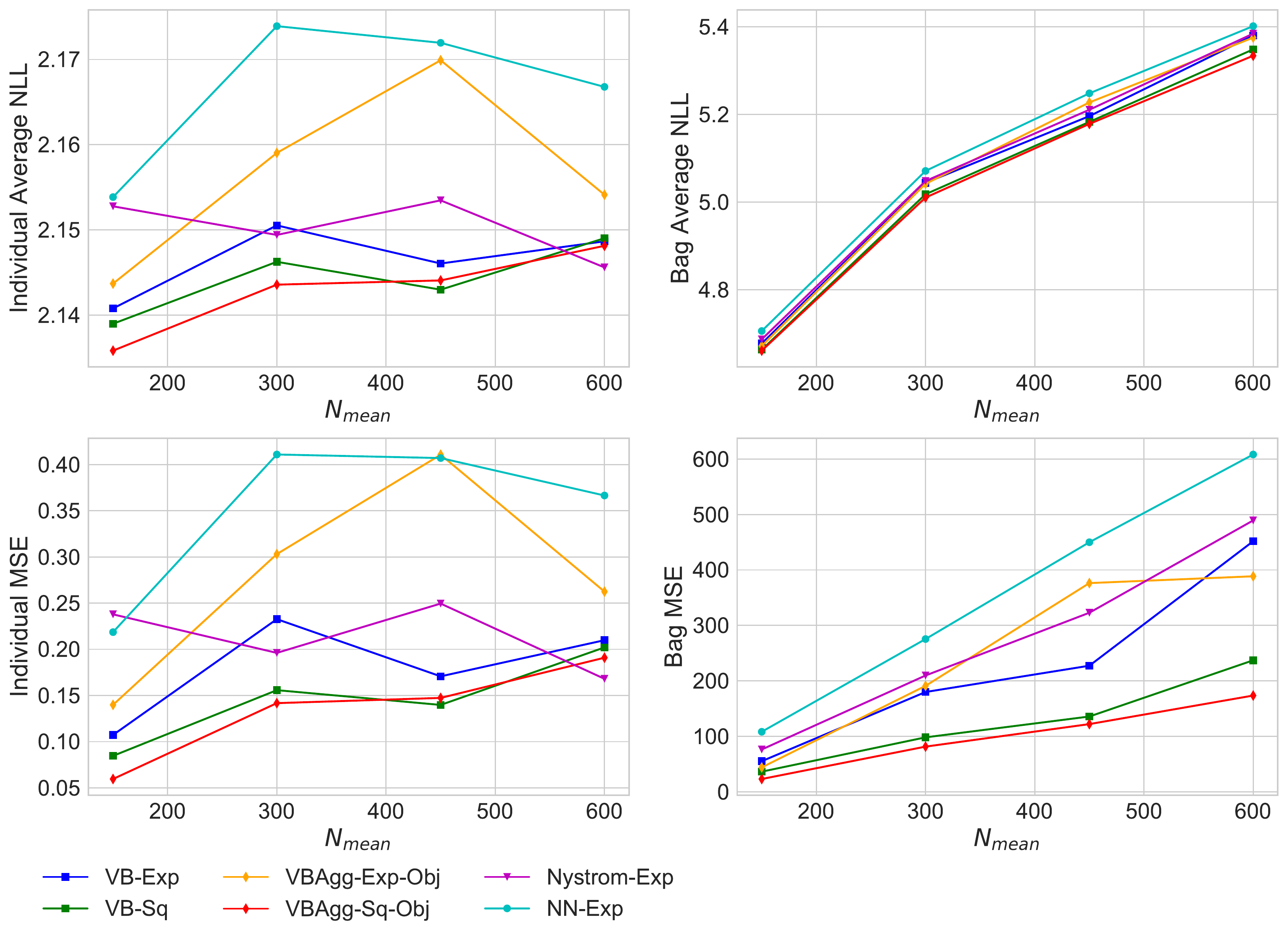}
\caption{Varying number of individuals per bag $N_{mean}$ over $5$ repetitions.\textbf{Left Column:} Individual average NLL and MSE on train set. \textbf{Right Column:} Bag average NLL and MSE on test set (of size $500$). Constant prediction NLL and MSE is $2.23$ and $0.85$ respectively.}
\label{fig:poisson_varyindiv}
\end{figure}
\begin{table}[H]
\centering
\caption{p-values from a Wilcoxon signed-rank test for VBAgg-Sq versus the methods below for the varying number of bags experiment for the Poisson model. The null hypothesis is VBAgg-Sq performs equal or worse than NN or  Nystr\"{o}m in terms of individual NLL or MSE on the train set.}
\label{tab:p_varybag_poisson_sq}
\begin{tabular}{lll}
        & NLL           & MSE           \\ \hline
NN-Exp      & $6.98\mathrm{e}{-06}$ & $0.00025$ \\
 Nystr\"{o}m-Exp & $0.00048$ & $0.00015$ \\ \hline
\end{tabular}
\end{table}
\begin{table}[H]
\centering
\caption{p-values from a Wilcoxon signed-rank test for VBAgg-Exp versus the methods below for the varying number of bags experiment for the Poisson model. The null hypothesis is VBAgg-Exp performs equal or worse than NN or  Nystr\"{o}m in terms of individual NLL or MSE on the train set.}
\label{tab:p_varybag_poisson_exp}
\begin{tabular}{lll}
        & NLL           & MSE           \\ \hline
NN-Exp      & $2.48\mathrm{e}{-06}$ & $2.48 \mathrm{e}{-05}$ \\
 Nystr\"{o}m-Exp & $0.0005$ & $0.00025$ \\ \hline
\end{tabular}
\end{table}

\begin{table}[H]
\centering
\caption{p-values from a Wilcoxon signed-rank test for VBAgg-Sq versus the methods below for the varying number of individuals per bag experiment for the Poisson model. The null hypothesis is VBAgg-Sq performs equal or worse than NN or  Nystr\"{o}m in terms of individual NLL or MSE on the train set.}
\label{tab:p_varyindiv_poisson_sq}
\begin{tabular}{lll}
        & NLL           & MSE           \\ \hline
NN-Exp      & $1.81\mathrm{e}{-05}$ & $9.53\mathrm{e}{-06}$ \\
 Nystr\"{o}m-Exp & $0.062$ & $0.041$ \\ \hline
\end{tabular}
\end{table}
\begin{table}[H]
\centering
\caption{p-values from a Wilcoxon signed-rank test for VBAgg-Exp versus the methods below for the varying number of individuals per bag experiment for the Poisson model. The null hypothesis is VBAgg-Exp performs worse than NN or  Nystr\"{o}m in terms of individual NLL or MSE on the train set.}
\label{tab:p_varyindiv_poisson_exp}
\begin{tabular}{lll}
        & NLL           & MSE           \\ \hline
NN-Exp      & $6.68\mathrm{e}{-05}$ & $0.00016$ \\
 Nystr\"{o}m-Exp & $0.049$ & $0.062$ \\ \hline
\end{tabular}
\end{table}
\paragraph{Calibration Plots for the Swiss Roll Dataset}
In Figure \ref{fig:poisson_calibrate_varybag} and \ref{fig:poisson_calibrate_varyindiv}, we provide calibration results for both experiments that we have considered. See top of Appendix \ref{app:experiments} for a further details. It is clear that while VBAgg-Sq-Obj and VBAgg-Sq provides good calibration in general, this is not the case for VBAgg-Exp-Obj and VBAgg-Exp. This is not surprising as the VBAgg-Exp methods uses an additional lower bound.
\FloatBarrier
\begin{figure}[h]
\includegraphics[width=\linewidth]{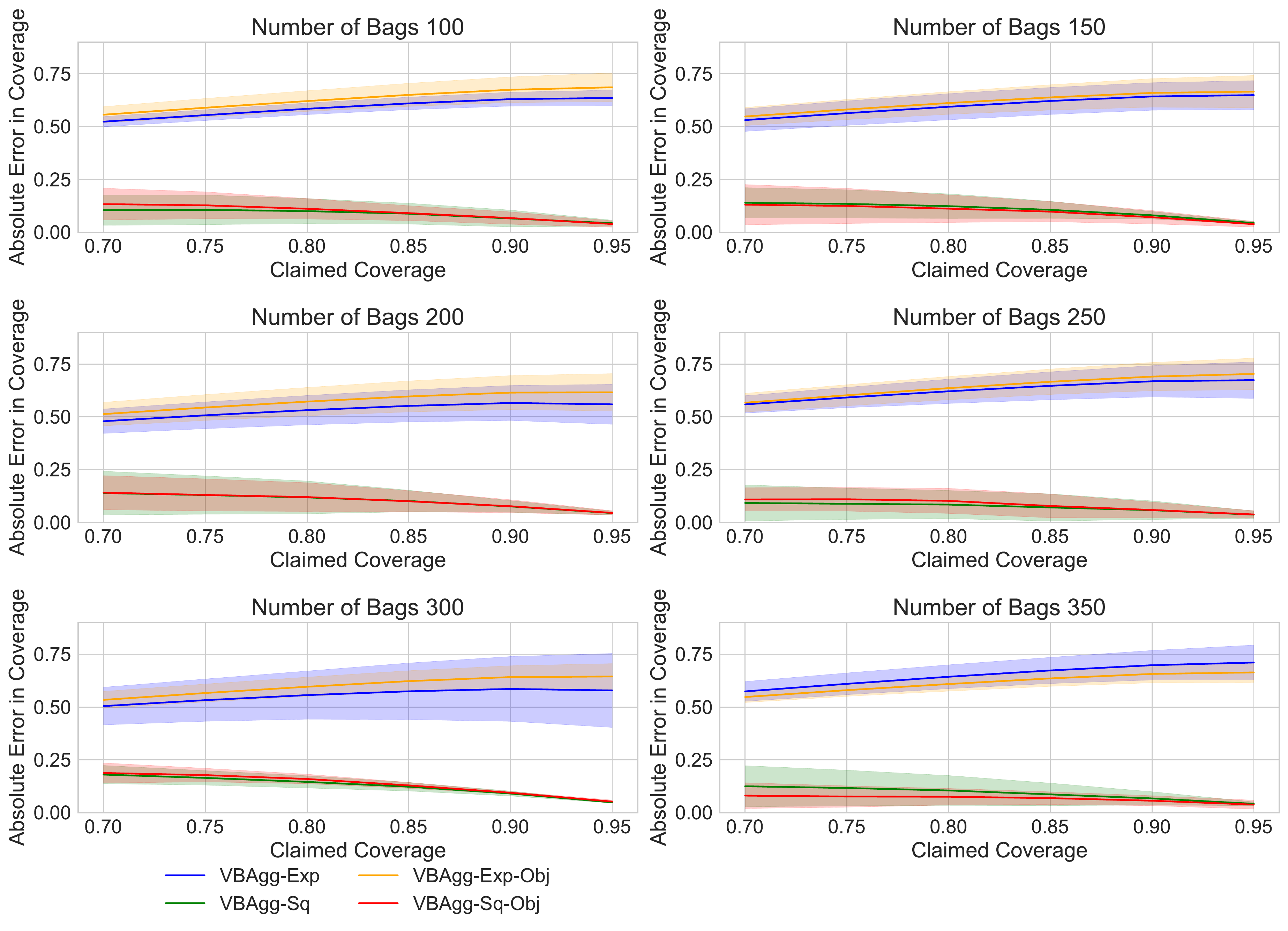}
\caption{Absolute Error in coverage from $70\%$ to $95\%$ for the increasing number of bags experiment for the Poisson Model. Shaded regions highlight the standard deviation. Perfect coverage would provide a straight line at $0$ error.}
\label{fig:poisson_calibrate_varybag}
\end{figure}
\begin{figure}[H]
\includegraphics[width=\linewidth]{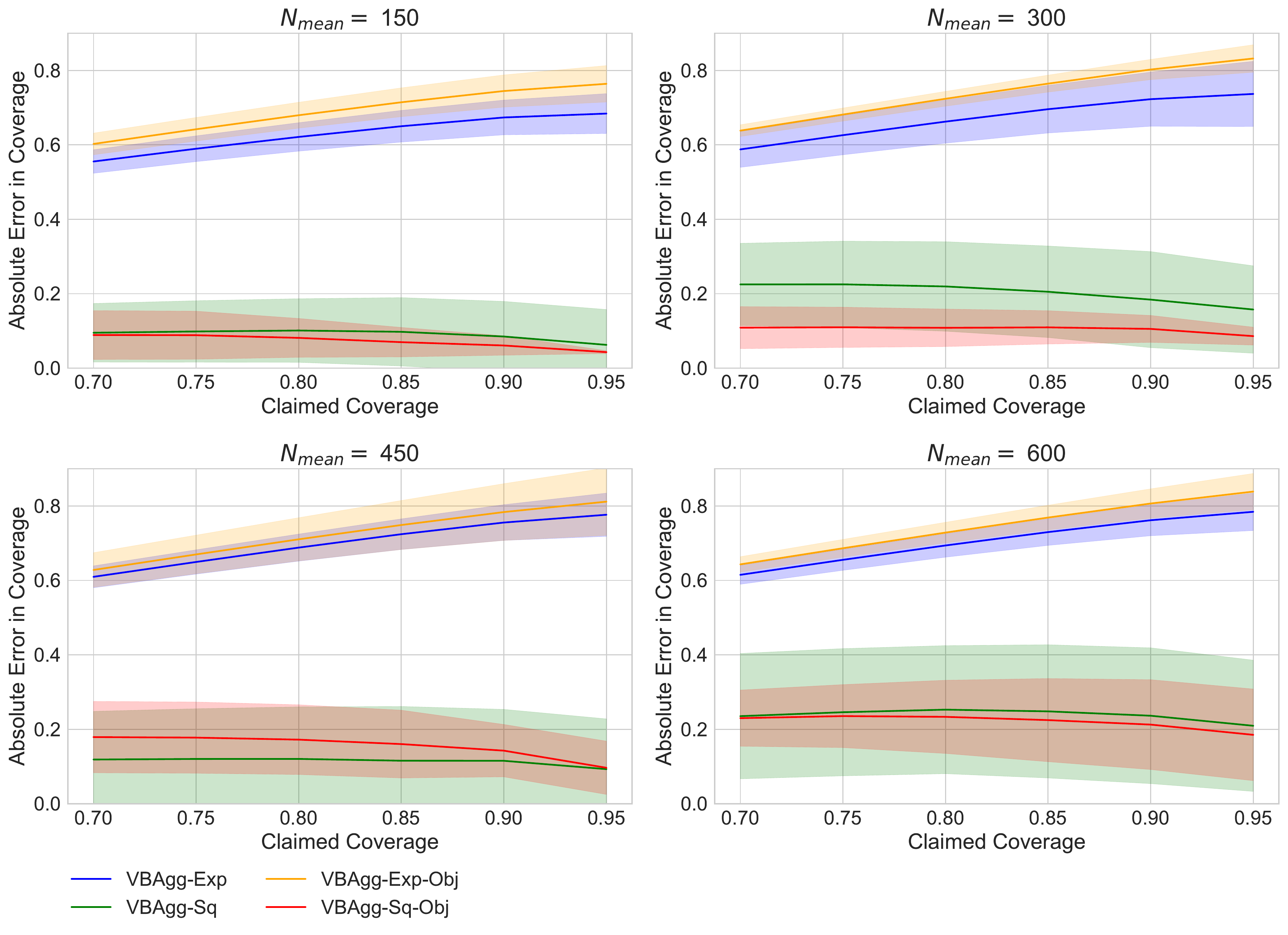}
\caption{Absolute Error in coverage from $70\%$ to $95\%$ for the increasing number of individuals per bag $N_{mean}$ and $N_{std}$ for the Poisson Model. Shaded regions highlight the standard deviation. Perfect coverage would provide a straight line at $0$ error.}
\label{fig:poisson_calibrate_varyindiv}
\end{figure}
\FloatBarrier
\paragraph{Prediction and uncertainty plots}
In Figure \ref{fig:poisson_pred_nn} and \ref{fig:poisson_uncer}, we provide some prediction plots for different models, and uncertainties for VBAgg models.

\begin{figure}[ht!]
\includegraphics[width=\linewidth]{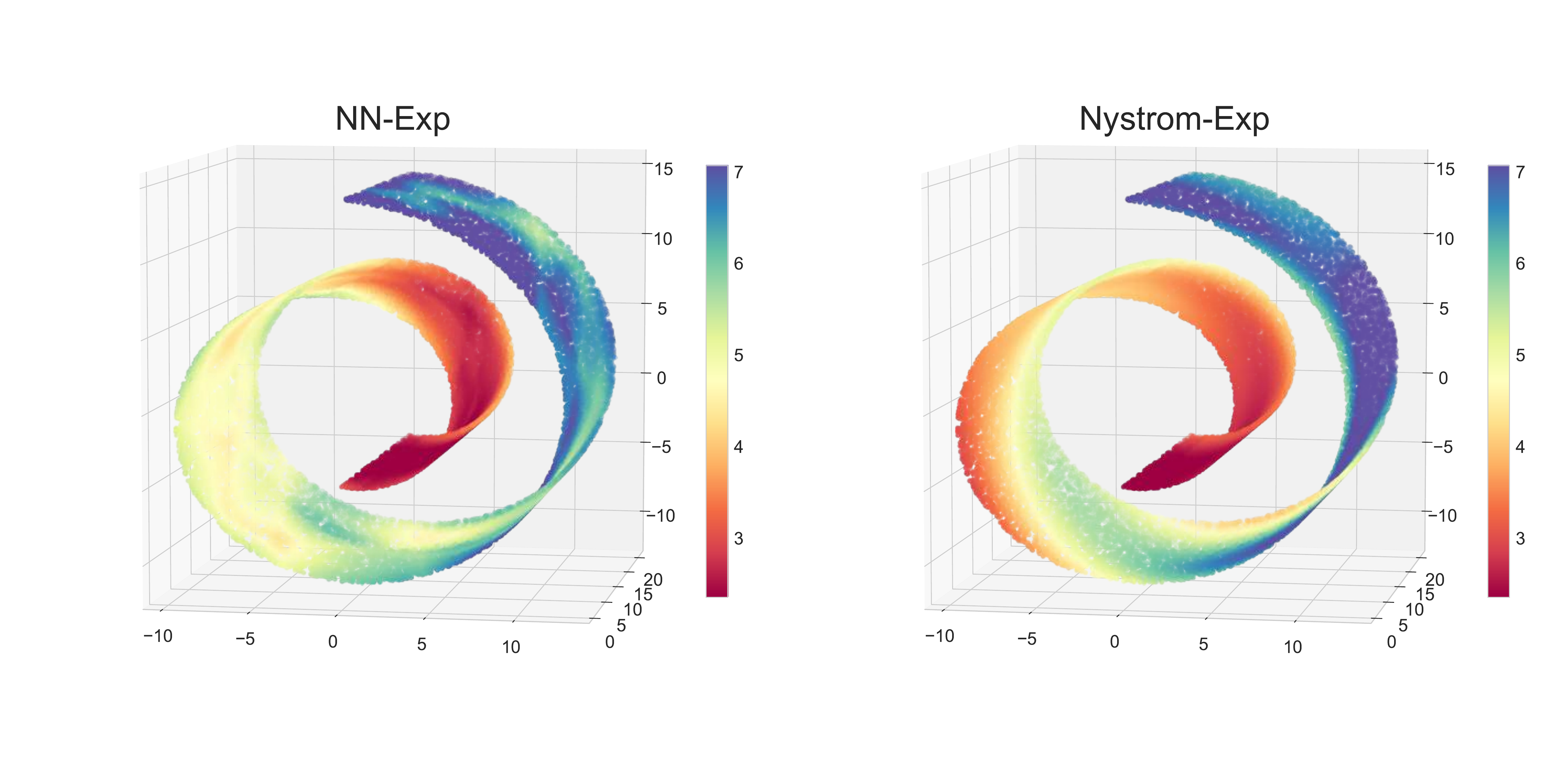}
\caption{Individual predictions on the train set for the swiss roll dataset with $150$ bags for NN and Nystr\"{o}m model. Here $N_{mean}=150$, with $N_{std}=50$.}
\label{fig:poisson_pred_nn}
\end{figure}
\begin{figure}[ht!]
\includegraphics[width=\linewidth]{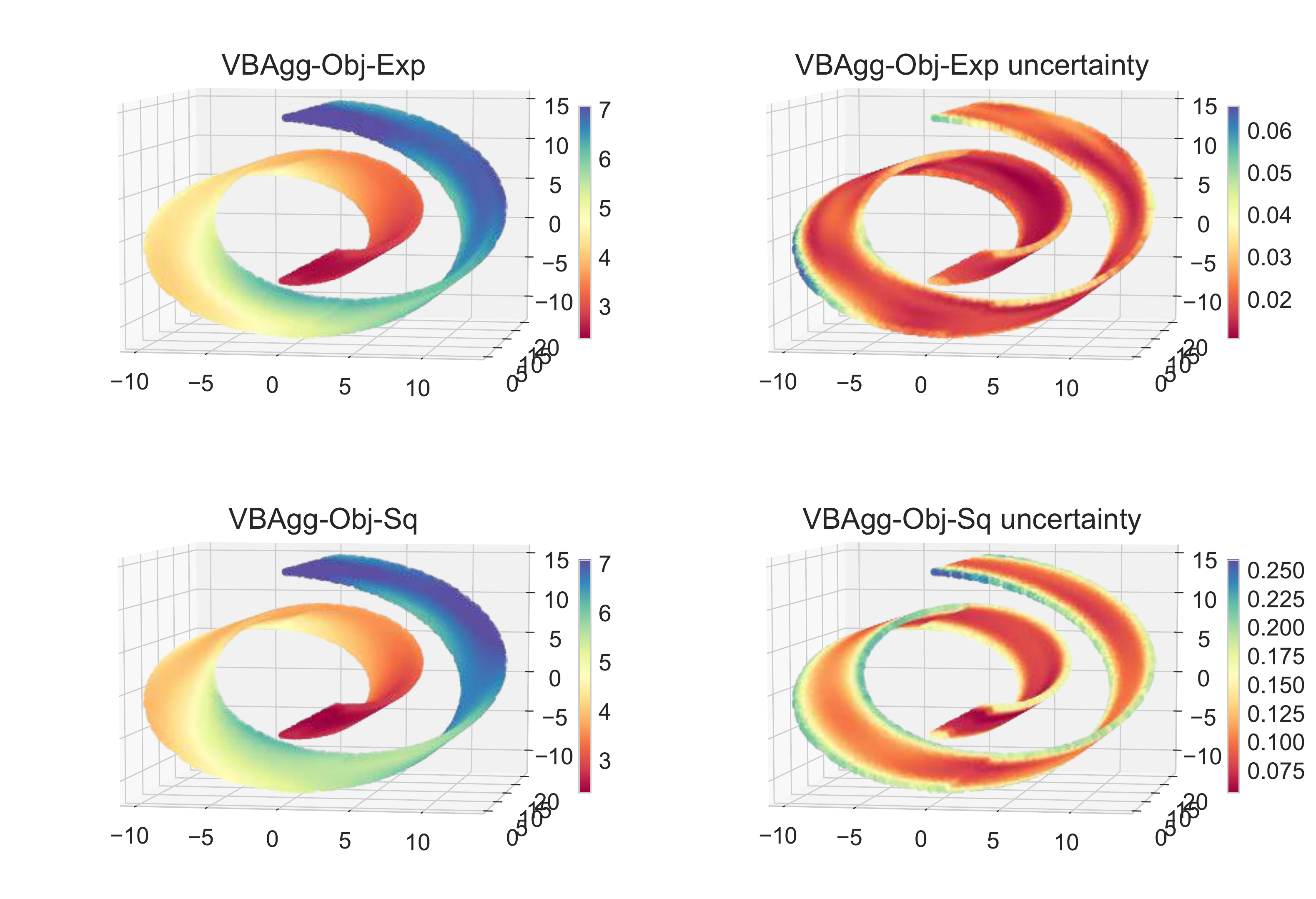}
\caption{Predictions and uncertainty on the swiss roll dataset with $150$ bags for the VBAgg-Obj models. Here $N_{mean}=150$, with $N_{std}=50$. For uncertainty, we plot the standard deviation of the posterior of $v$, coming from $v^a \sim \mathcal{N}(m^a, S^a)$ in (\ref{eqn:meanCov}).}
\label{fig:poisson_uncer}
\end{figure}
\FloatBarrier

\subsection{Normal Model}
\label{app:normal_exp}
\subsubsection{Swiss Roll Dataset}
In this section, we provide some experimental results for the Normal model, where throughout we assume $\tau^a_i = \tau$, same for all individuals. 

We consider the same swiss roll dataset as in the Poisson model, here the colour of each point to be the underlying mean $\mu^a_i$. We then consider $y^a_i \sim \mathcal{N}(\mu^a , \tau)$ with $\tau=0.1$, hence bag observations are given by $y^a = \sum_{i=1}^{N_a} y^a_i  \sim \mathcal{N}(\mu^a , N_a \tau)$ with $\mu^a=\sum_{i=1}^{N_a} \mu^a_i$. Here, the goal is to predict $\mu^a_i$ and $\tau$, given bag observations $y^a$ only. The results for the experiments are shown below in Figure \ref{fig:normal_varybag} and Figure \ref{fig:normal_varyindiv}, which shows the VBAgg outperforming the NN and Nystr\"om model. To show statistical significance, we also report 
 the corresponding table of p-values in Table \ref{tab:p_varybag_normal} and Table \ref{tab:p_varyindiv_normal}. Furthermore, we would also like to point out that the VBAgg is well calibrated as shown in Figure \ref{fig:normal_calibrate_varybag}.
\FloatBarrier
\begin{figure}[ht!]
\includegraphics[width=0.8\linewidth]{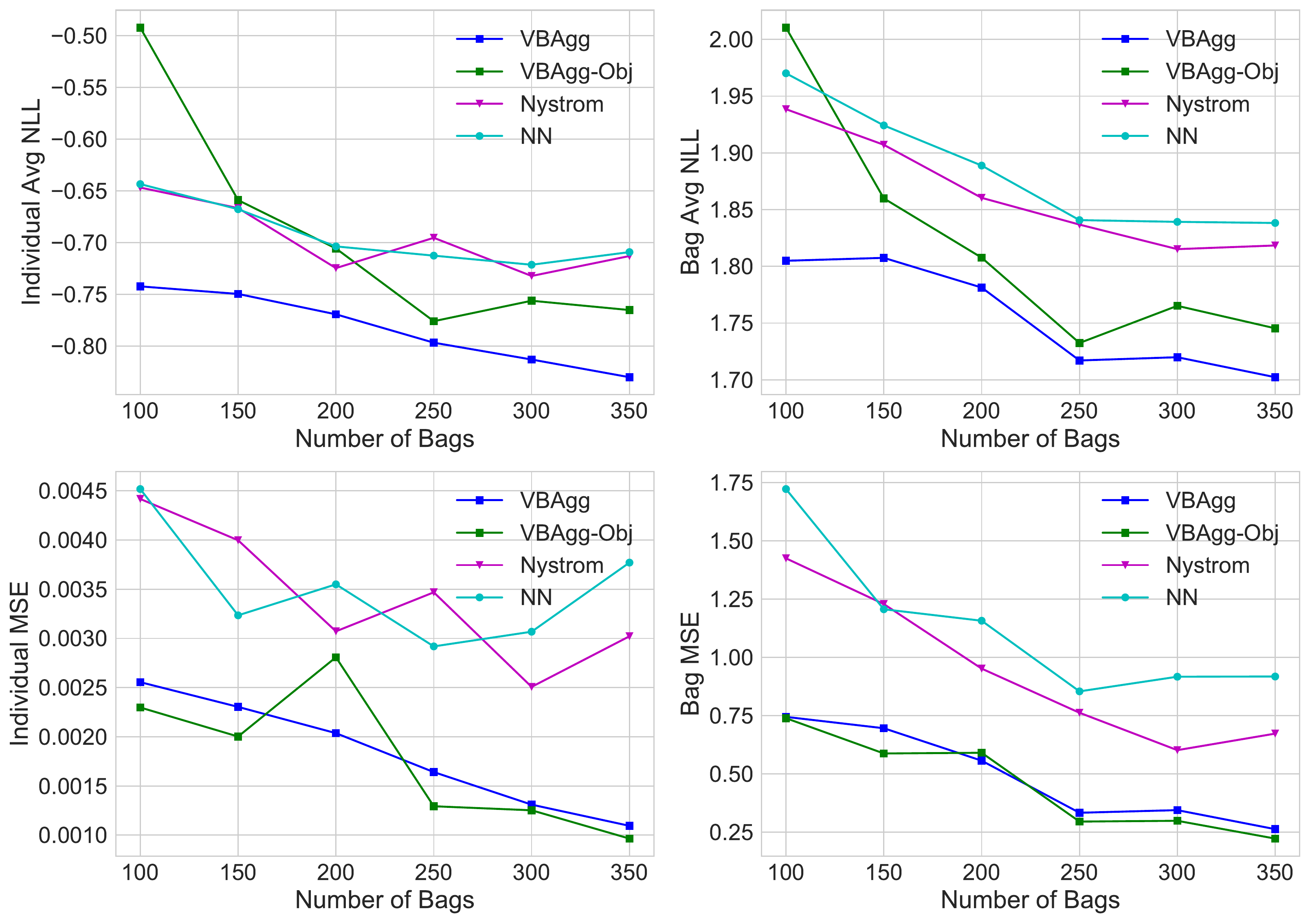}
\caption{Varying number of bags over $5$ repetitions for the Normal model.\textbf{Left Column:} Individual average NLL and MSE on train set. \textbf{Right Column:} Bag average NLL and MSE on test set (of size $500$). Constant model individual MSE is $0.04$.}
\label{fig:normal_varybag}
\end{figure}
\begin{table}[ht!]
\centering
\caption{p-values from a Wilcoxon signed-rank test for VBAgg versus the methods below for the varying number of bags experiment for the Normal model. The null hypothesis is VBAgg performs equal or worse than NN or  Nystr\"{o}m in terms of individual NLL or MSE on the train set.}
\label{tab:p_varybag_normal}
\begin{tabular}{lll}
        & NLL           & MSE           \\ \hline
NN      & $5.96\mathrm{e}{-07}$ & $4.79\mathrm{e}{-09}$ \\
 Nystr\"{o}m & $4.01\mathrm{e}{-08}$ & $6.52\mathrm{e}{-09}$ \\ \hline
\end{tabular}
\end{table}
\begin{figure}
\includegraphics[width=0.8\linewidth]{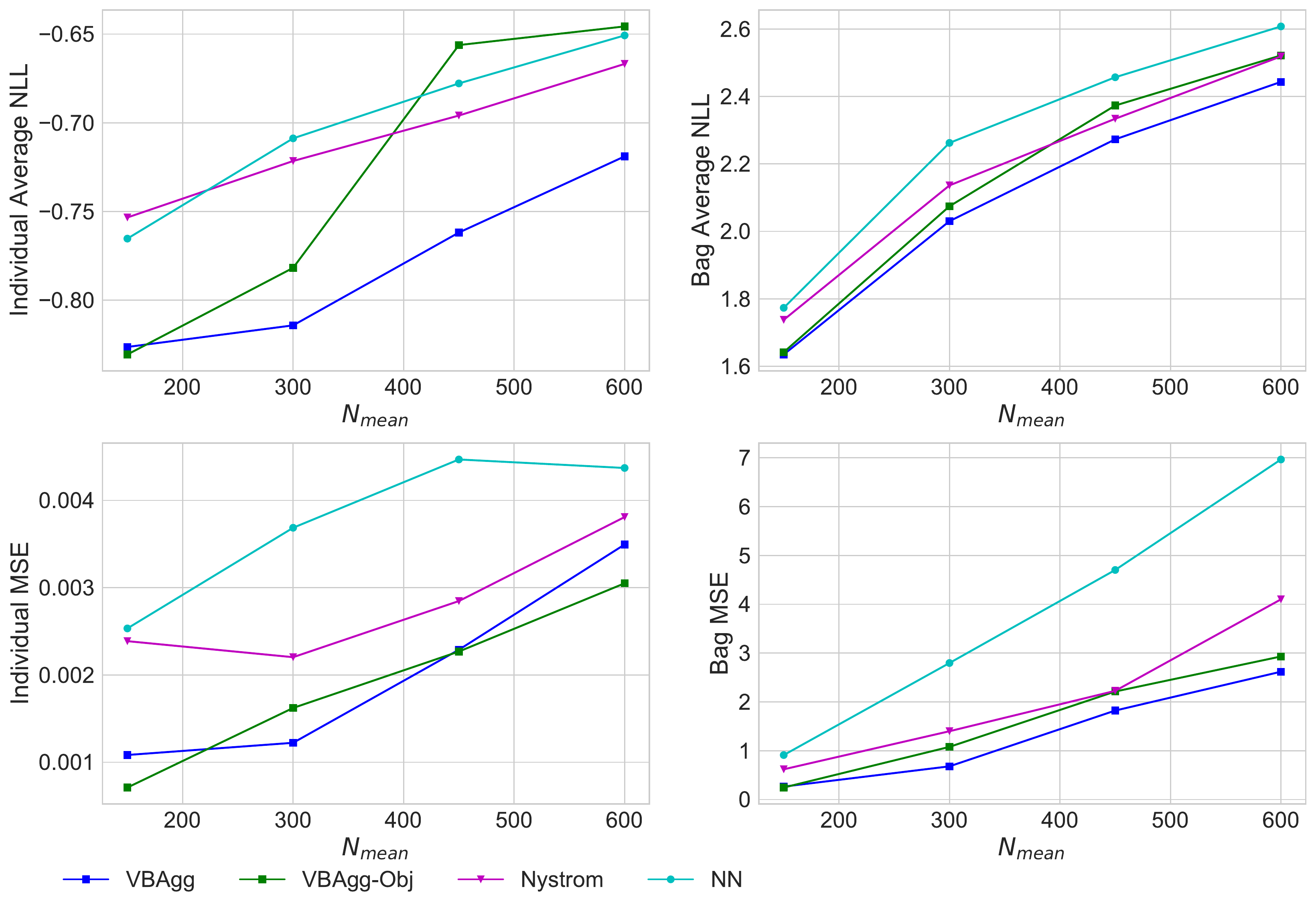}
\caption{Varying number of individuals per bag $N_{mean}$ over $5$ repetitions.\textbf{Left Column:} Individual average NLL and MSE on train set. \textbf{Right Column:} Bag average NLL and MSE on test set (of size $500$). Constant model individual MSE is $0.039$.}
\label{fig:normal_varyindiv}
\end{figure}
\begin{table}[ht!]
\centering
\caption{p-values from a Wilcoxon signed-rank test for VBAgg versus the methods below for the varying number of individuals per bag $N_{mean}$ experiment for the Normal nodel. The null hypothesis is VBAgg performs worse than NN or  Nystr\"{o}m in terms of individual NLL or MSE on the train set.}
\label{tab:p_varyindiv_normal}
\begin{tabular}{lll}
        & NLL           & MSE           \\ \hline
NN      & $4.77\mathrm{e}{-06}$ & $4.77\mathrm{e}{-06}$ \\
 Nystr\"{o}m & $4.77\mathrm{e}{-06}$ & $4.77\mathrm{e}{-06}$ \\ \hline
\end{tabular}
\end{table}
\paragraph{Calibration Plots for the Swiss Roll Dataset}
In Figure \ref{fig:normal_calibrate_varybag} and \ref{fig:normal_calibrate_varyindiv}, we provide calibration results for both experiments that we have considered. See top of Appendix \ref{app:experiments} for further details. It is clear that VBAgg-Obj has better calibration in general, this is not surprising as it is tuned based on the correct objective, rather than NLL.
\begin{figure}
\includegraphics[width=0.9\linewidth]{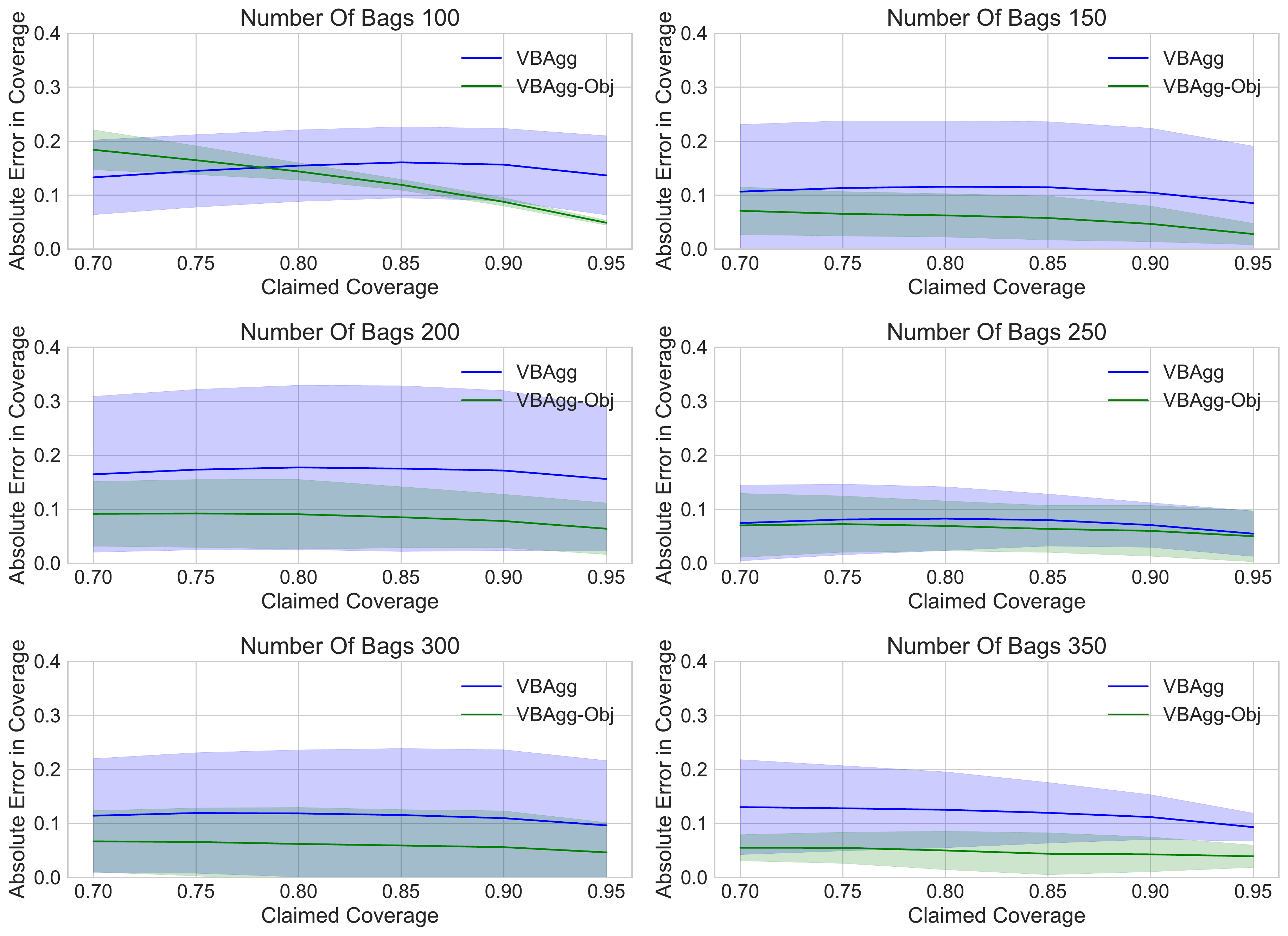}
\caption{Absolute Error in coverage from $70\%$ to $95\%$ for the increasing number of bags experiment for the Normal Model. Shaded regions highlight the standard deviation. Perfect coverage would provide a straight line at $0$ error.}
\label{fig:normal_calibrate_varybag}
\end{figure}
\begin{figure}
\includegraphics[width=0.9\linewidth]{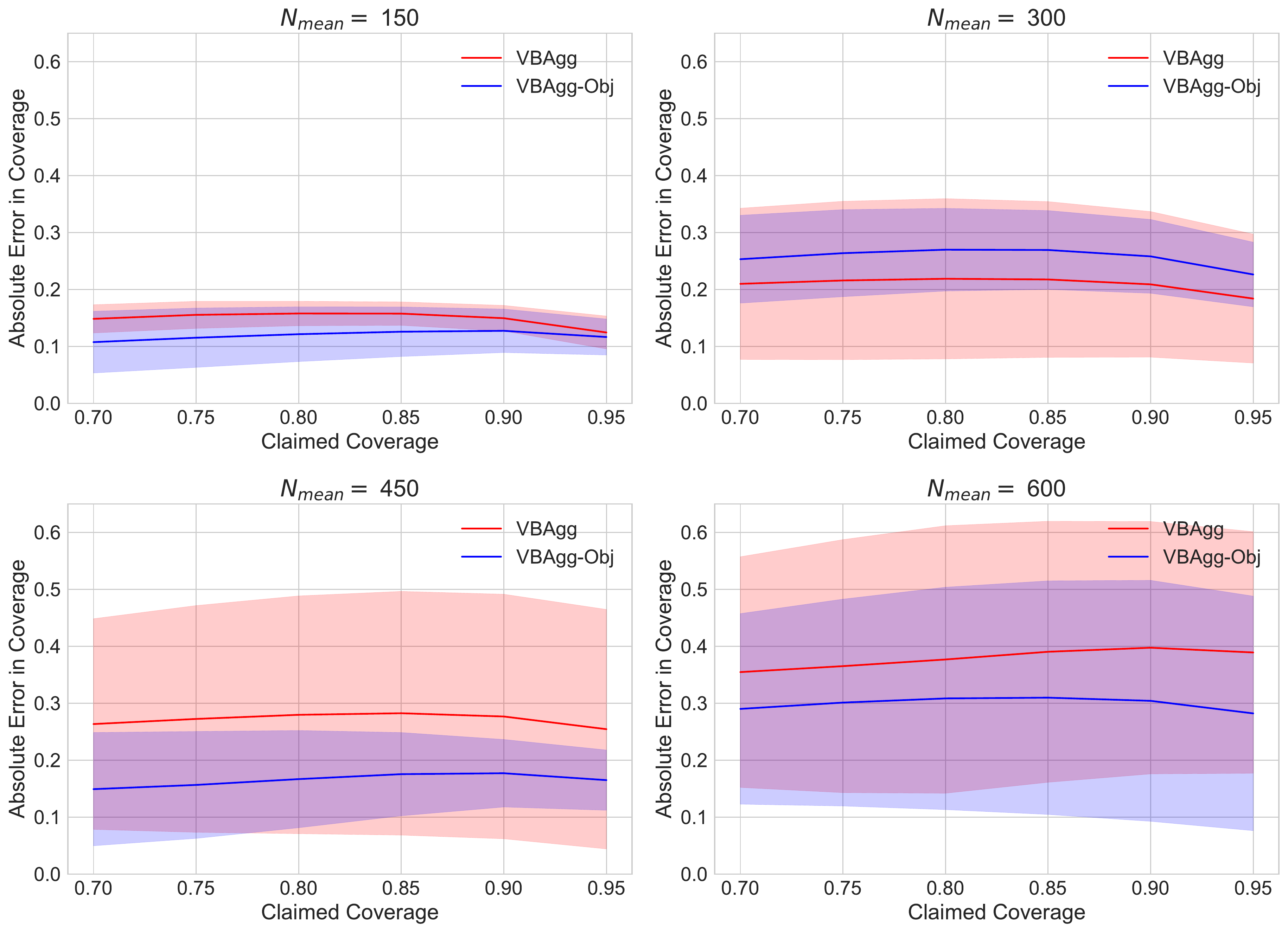}
\caption{Absolute Error in coverage from $70\%$ to $95\%$ for the increasing number of individuals per bag $N_{mean}$ and $N_{std}$ for the Normal Model. Shaded regions highlight the standard deviation. Perfect coverage would provide a straight line at $0$ error.}
\label{fig:normal_calibrate_varyindiv}
\end{figure}
\FloatBarrier
\paragraph{Prediction and uncertainty plots}
Here, we provide some prediction plots for different models.
\begin{figure}[h]
\includegraphics[width=0.7\linewidth]{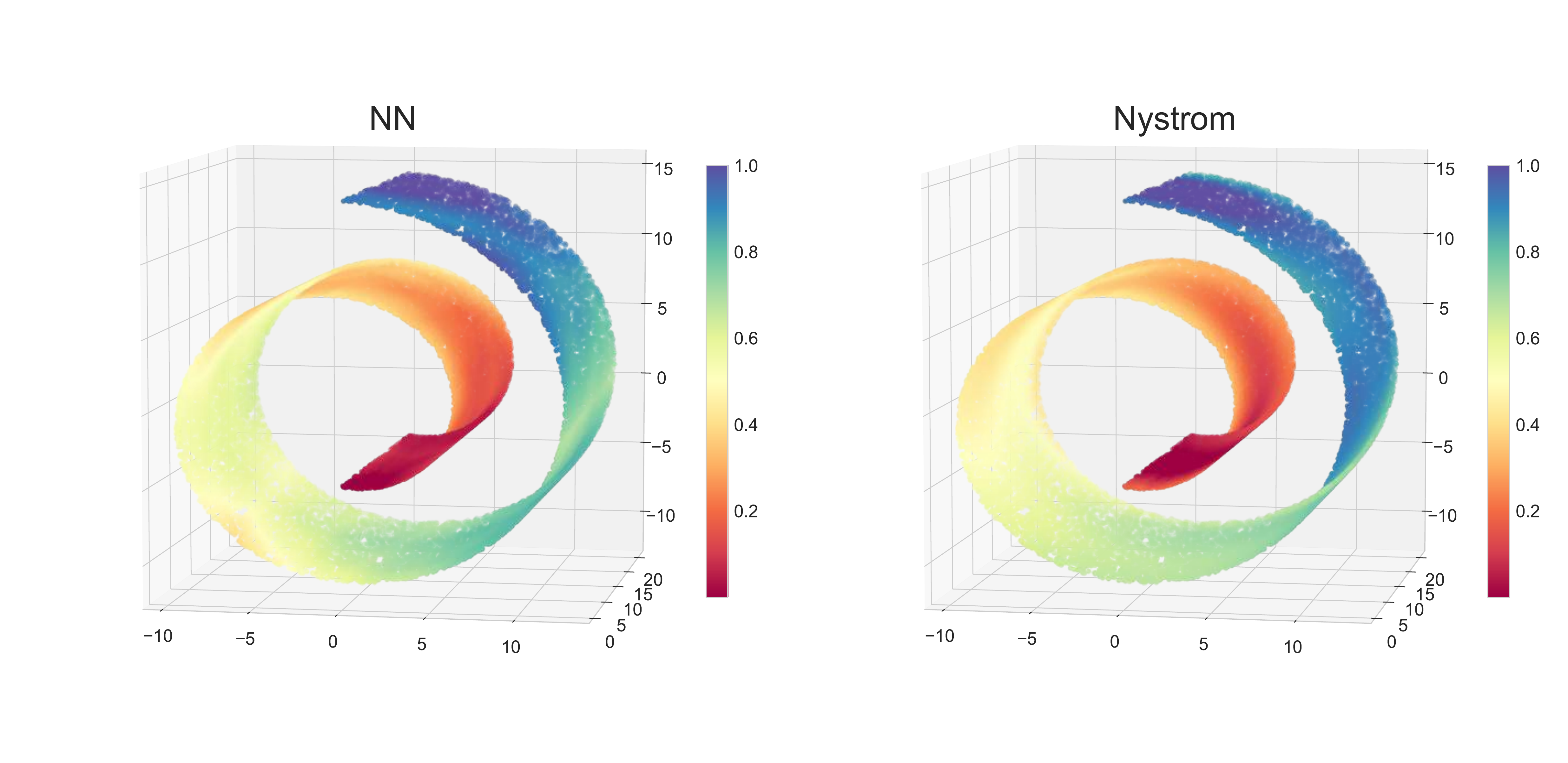}
\caption{Individual predictions on the train set for the swiss roll dataset with $150$ bags for NN and Nystr\"{o}m model. Here $N_{mean}=150$, with $N_{std}=50$.}
\label{fig:normal_prediction_nn}
\end{figure}
\begin{figure}[h]
\includegraphics[width=0.7\linewidth]{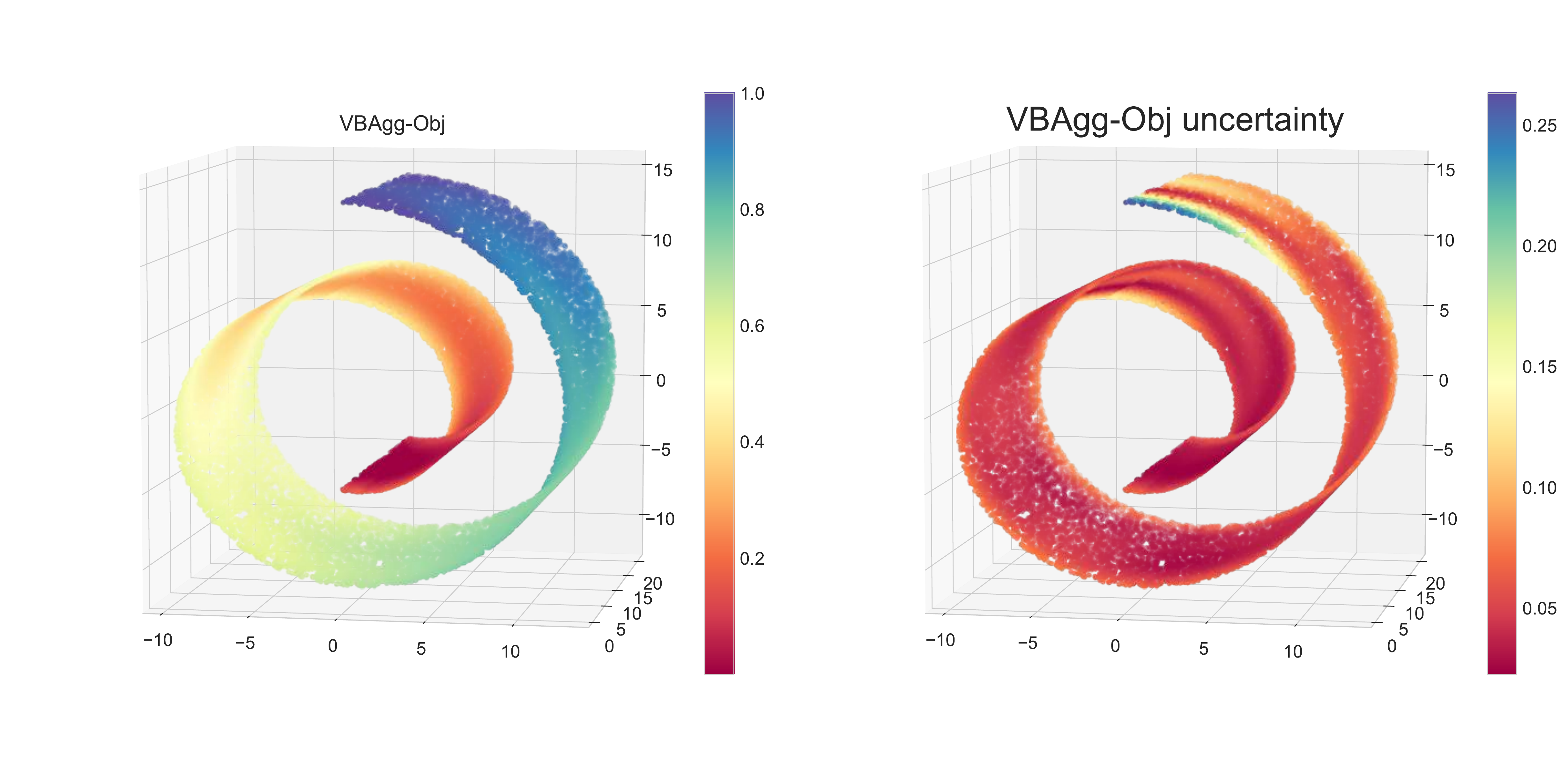}
\caption{Predictions and uncertainty on the swiss roll dataset with $150$ bags for the VBAgg-Obj model. Here $N_{mean}=150$, with $N_{std}=50$. For uncertainty, we plot the standard deviation of the posterior of $v$, coming from $v^a \sim \mathcal{N}(m^a, S^a)$ in (\ref{eqn:meanCov}).}
\label{fig:normal_uncer}
\end{figure}
\FloatBarrier
\subsubsection{Elevators Dataset}
For a real dataset experiment, we consider the elevators dataset\footnote{This dataset is publicly available at http://sci2s.ugr.es/keel/dataset.php?cod=94}, which is a large scale regression dataset\footnote{We have removed one column that is almost completely sparse.} containing $16599$ instances, with each instance $\in \mathbb{R}^{17}$. This dataset is obtained from the task of controlling F16 aircraft, with the label $y$ being a particular action taken on the elevators of the aircraft $\in \mathbb{R}$. For the model formulation we assume each label follows a normal distribution, i.e. $y_l \sim \mathcal{N}(\mu_l, \tau)$, where $\tau$ is a fixed quantity to be learnt. In practice, we can imagine the action taken may differ according to the operator. \\

In order formulate this dataset in an aggregate data setting, we sample bag sizes from a negative binomial distribution as before, with $N_{mean}=30$ and $N_{std}=15$, and also take $w^a_i = 1$. To place observations into bags, similar to the swiss roll dataset, we consider a particular covariate, and place instances into bags based on the ordering of the covariate. We now have the bag-level model given by $y^a \sim \mathcal{N}(\mu^a, N_a\tau)$, with individual model $y^a_i \sim \mathcal{N}(\mu^a_i, \tau)$ and it is our goal to predict $\mu^a_i$ (and also infer $\tau$), given only $y^a$. After the bagging process, we obtain approximately $225$ bags for training, and $33$ bags each for early stopping, validation and testing (for bag level performance). Further, in order to neglect variables that do not provide signal, we use an ARD kernel for the VBAgg and Nystr\"om model, as below:
\begin{equation}
\label{eqn:ard_k}
k_{ard}(x,y) = \gamma_{scale} \exp\left(-\frac{1}{2}\sum_{k=1}^d \frac{1}{\ell_k}(x_k - y_k)^2 \right)
\end{equation}
and learn kernel parameters $\gamma_{scale}$ and $\{\ell_k\}_{k=1}^d$. We repeat this process and splitting of the dataset $50$ times and report individual NLL results, and also MSE results in Table \ref{tab:elevator_full}. From the results, we observe that the VBAgg model performs better the Nystr\"om and NN model, with statistical significance.
\begin{table}[t]
\centering
\caption{Results for the Normal Model on the elevators dataset with $50$ repetitions. Indiv represents individuals on train set here, while bag performance is measured on a test set. Numbers in brackets denotes p-values from a Wilcoxon signed-rank test for VBAgg versus the method. The null hypothesis is VBAgg performs equal or worse than NN or  Nystr\"{o}m in terms of individual NLL or MSE on the train set. It is also noted MSE is computed on the observed $y^a_i$ or $y^a$, rather than the unknown $\mu^a_i$ or $\mu^a$.}
\label{tab:elevator_full}
\begin{tabular}{lllll}
          & Indiv NLL & Bag NLL & Indiv MSE & Bag MSE \\ \hline
Constant  & N/A            & N/A     & 0.010         & 0.366   \\
VBAgg     & -1.69   & 0.003   & 0.0018         & 0.052   \\
VBAgg-Obj & -1.71   & -0.02   & 0.0018         & 0.052   \\
Nystr\"{o}m   & $-1.57 (1.5\mathrm{e}{-13})$  & 0.003   & 0.0024 $(8.9\mathrm{e}{-16})$    & 0.041   \\
NN        & -1.64 (0.0001258) & 0.082   & 0.0021 $(8.8\mathrm{e}{-10})$         & 0.041   \\ \hline
\end{tabular}
\end{table}

\end{document}